\documentclass[11pt]{article}

\usepackage[T1]{fontenc}
\usepackage[utf8]{inputenc}
\usepackage[letterpaper,margin=1in]{geometry}
\usepackage[hyphens]{url}
\usepackage{graphicx}
\usepackage[round,authoryear]{natbib}
\usepackage{caption}
\usepackage{subcaption}
\usepackage{placeins}
\usepackage{float}
\usepackage{tikz}
\usetikzlibrary{arrows.meta,calc,shapes.geometric,decorations.pathreplacing}
\usepackage{booktabs}
\usepackage{amsmath,amssymb,amsthm,mathtools}
\usepackage{tabularx}
\usepackage{multirow}
\usepackage{microtype}
\usepackage{hyperref}
\hypersetup{
  colorlinks=true,
  linkcolor=blue,
  citecolor=blue,
  urlcolor=blue,
  pdftitle={Feed-Forward Steering in Transformer Residual Dynamics},
  pdfauthor={Timur Mudarisov, Mikhail Burtsev, Radu State}
}

\newtheorem{theorem}{Theorem}
\newtheorem{proposition}{Proposition}
\newtheorem{lemma}{Lemma}
\newtheorem{corollary}{Corollary}

\newcommand{\R}{\mathbb{R}}
\newcommand{\Sph}{\mathbb{S}}
\newcommand{\Prob}{\mathcal{P}}
\newcommand{\E}{\mathbb{E}}
\newcommand{\norm}[1]{\left\lVert #1 \right\rVert}
\newcommand{\ip}[2]{\left\langle #1,#2 \right\rangle}
\newcommand{\Proj}{\mathrm{P}}

\newcommand{\divS}{\mathrm{div}_{\Sph^{d-1}}}
\newcommand{\Spec}{\mathrm{Spec}}
\newcommand{\RePart}{\mathrm{Re}}

\newcommand{\id}{\mathrm{Id}}

\tikzset{
  tok/.style={circle,fill=black,inner sep=1.1pt},
  mem/.style={star,star points=5,star point ratio=2.3,fill=orange!85!red,draw=black,inner sep=1.4pt},
  attr/.style={circle,fill=blue!70!black,draw=black,inner sep=1.5pt},
  src/.style={circle,draw=black,fill=white,inner sep=1.5pt},
  fld/.style={-{Stealth[length=1.7mm]},gray!70,thin},
  big/.style={-{Stealth[length=2.6mm]},very thick},
  sphere/.style={draw=black!55,thick},
  lab/.style={font=\small},
}

\title{Feed-Forward Steering in Transformer Residual Dynamics}
\author{%
Timur Mudarisov$^{1}$ \quad Mikhail Burtsev$^{2}$ \quad Radu State$^{1}$\\[0.6em]
\small $^{1}$University of Luxembourg \qquad
$^{2}$London Institute of Mathematical Sciences
}
\date{}

\begin{document}
\maketitle

\begin{abstract}
Attention-only dynamical theories model Transformer residual directions as particles aggregating on a sphere. We extend this framework by incorporating the feed-forward network (FFN) term as a local steering field acting on each token state. The resulting theory predicts that the tangential component of the FFN field is necessary for motion in residual-direction space, that critical residual directions correspond to nonlinear projective equilibria, and that a direct sequential-to-parallel defect characterizes when a finite attention--FFN block can be accurately approximated by a parallel, additive map. Across GPT-2, Pythia, Mistral, and Llama models, the attention branch alone exhibits a substantial angular-alignment deficit relative to the realized block update, with the deficit increasing from GPT-2 to Llama-3-8B. Intervention experiments show that retaining only the tangential FFN component preserves most model quality, whereas retaining only the radial component causes performance to collapse. The tangential component also preserves output diversity under aggregation pressure. As a practical application, layers with small sequential-to-parallel defects can be approximately parallelized with only a modest increase in loss, whereas layers with large defects degrade rapidly. These findings support the interpretation of FFN layers as directional steering fields that shape Transformer residual geometry and govern the feasibility of block-level interventions.
\end{abstract}

\section{Introduction}
A transformer block combines a non-local mechanism represented by attention and a position-wise mechanism represented by fully connected layers (FFN). Existing continuous attention dynamics theories treat token directions as
interacting particles and explain clustering or rank collapse through aggregation
\citep{geshkovski2023mathematical,geshkovski2023emergence,karagodin2024causal}. In
parallel, interpretability and editing work show that FFNs store key--value memories
and are effective editing sites \citep{geva2021transformer,geva2022transformer,meng2022rome,meng2023memit}. We connect these views by making the FFN a term in the model of residual stream dynamics.

In our study we consider attention as aggregation, and FFN as steering.
Attention supplies a field $A_i(U)$ depending on other tokens, and the FFN supplies a
local field $\Phi_i(u_i)$ depending on the current token. Normalization makes
residual direction the natural state variable, while the residual magnitude
controls angular speed through the factor $1/r_i$
\citep{xiong2020layernorm,zhang2019rmsnorm,karagodin2025normalization}. The resulting model is
\begin{equation}
\dot u_i=\frac1{s_i(t)}P_{u_i}\big[A_i(U)+\Phi_i(u_i)\big].
\label{eq:flow}
\end{equation}
Here $s_i(t)>0$ is a prescribed angular-speed scale. When this sphere-valued
model is derived from an ambient residual dynamics $x_i=r_i u_i$, the exact
choice is $s_i(t)=r_i(t)=\|x_i(t)\|$. In the normalized-time models below we
set $s_i\equiv1$; this is a normalized-time direction-only approximation when
token norms vary. Setting $\Phi=0$ recovers attention-only particle dynamics, and nonzero
$\Phi$ turns it into an aggregation--steering system (Fig.~\ref{fig:concept-directions-a}). 
%Empirically, the one-step angular gain of the steering term increases with model scale, which suggests that the effect is not a small-model artifact but a stronger feature of larger decoder LLMs.

\begin{figure*}[!t]
\centering

\begin{subfigure}[t]{0.64\textwidth}
\centering
\resizebox{\textwidth}{!}{%
\begin{tikzpicture}[scale=0.9]

% tuning parameters for horizontal spacing
\def\sep{6.3}        % distance between neighboring circles
\def\halfsep{3.15}   % half of that, for + and = positions

% Panel 1: aggregation
\begin{scope}
  \node[lab] at (0,2.3) {\textbf{Aggregation} (attention)};
  \draw[sphere] (0,0) circle (1.55);
  \def\C{55}
  \draw[big,red!70!black] (0,0) -- (\C:1.55) node[above right]{$Vu$};
  \foreach \a in {10,30,80,120,150,200,250,300,330}{
    \node[tok] at (\a:1.55) {};
    \draw[fld,red!55] (\a:1.55) .. controls (\a:1.1) and (\C:1.1) .. (\C:1.4);}
  \node[lab,align=center] at (0,-2.2)
    {\scriptsize non-local pull toward\\ \scriptsize shared directions};
\end{scope}

\node at (\halfsep,0) {\Large $+$};

% Panel 2: steering
\begin{scope}[shift={(\sep,0)}]
  \node[lab] at (0,2.3) {\textbf{Steering} (FFN)};
  \draw[sphere] (0,0) circle (1.55);
  \foreach \w/\nm in {40/$w_1$,150/$w_2$,300/$w_3$}{
    \node[mem] at (\w:1.55) {};
    \node[lab] at (\w:1.9){\nm};
    \foreach \d in {-22,-11,11,22}{
      \pgfmathtruncatemacro\aa{\w+\d}
      \node[tok] at (\aa:1.55){};
      \draw[fld,orange!80!black] (\aa:1.55)
        .. controls (\aa:1.15) and (\w:1.15) .. (\w:1.45);
    }}
  \node[lab,align=center] at (0,-2.3)
    {\scriptsize local push toward\\ \scriptsize FFN value directions};
\end{scope}

\node at ({\sep+\halfsep},0) {\Large $=$};

% Panel 3: combined
\begin{scope}[shift={({2*\sep},0)}]
  \node[lab] at (0,2.3) {\textbf{Residual flow}};
  \draw[sphere] (0,0) circle (1.55);
  \node[attr] at (50:1.55){};
  \node[lab] at (50:1.9){\scriptsize attractor};
  \node[attr] at (175:1.55){};
  \node[src] at (300:1.55){};
  \node[lab,align=center] at (300:2.0){\scriptsize saddle direction};
  \foreach \a/\n in {20/50,70/50,120/175,150/175,210/175,250/175,330/50,10/50}{
     \draw[fld] (\a:1.55) .. controls (\a:1.2) and (\n:1.2) .. (\n:1.45);}
  \node[lab,align=center] at (0,-2.3)
    {\scriptsize zeros of $g$ $=$ critical directions\\
     \scriptsize stable $=$ residual attractors};
\end{scope}

\end{tikzpicture}}
\caption{\textbf{Aggregation--steering view of residual flow.}}
\label{fig:concept-directions-a}
\end{subfigure}
\hfill
\begin{subfigure}[t]{0.33\textwidth}
\centering
\resizebox{\textwidth}{!}{%
\begin{tikzpicture}[scale=0.95, every node/.style={font=\small}]
  \draw[gray!70, thick] (0,0) circle (1.55);
  \node at (-0.45,1.86) {unit sphere $\Sph^{d-1}$};

  \coordinate (O) at (0,0);
  \coordinate (X) at (38:1.55);
  \coordinate (R) at ($(X)+(0.90,0.65)$);
  \coordinate (F) at ($(X)+(1.20,0.18)$);
  \coordinate (T) at ($(X)+(0.55,-0.55)$);

  \draw[->, thick] (O) -- (X) node[pos=0.58, above left] {$x=ru$};
  \fill (X) circle (1.2pt) node[above left] {$u$};

  \draw[->, very thick, black!70] (X) -- (R)
    node[above right] {$\langle u,F\rangle u$ (radial)};
  \draw[->, very thick, red!80!black] (X) -- (F)
    node[right] {$F=Vu+\alpha\Phi$};
  \draw[->, very thick, blue!80!black] (X) -- (T)
    node[below right] {$g=P_uF$ (tangent)};

  \draw[->, blue!65!black, thick] ($(X)+(0.10,-0.05)$)
    .. controls ($(X)+(0.15,-0.35)$) and ($(X)+(0.35,-0.45)$) .. ($(X)+(0.52,-0.50)$);

  \node[align=center, font=\small] at (0,-2.35)
  {Only the \textbf{tangential} part $g=P_uF$ turns $u$;\\
   the radial part rescales $r$. Magnitude sets angular \emph{speed} only.};
\end{tikzpicture}}
\caption{\textbf{Why directions.}}
\label{fig:concept-directions-b}
\end{subfigure}

\caption{\textbf{Residual-direction dynamics.}
\textbf{(a)} Attention is a masked non-local aggregation field. The FFN is a local steering field whose value directions can create, move, or destabilize critical residual directions. The combined field $g(u)=P_u(Vu+\Phi(u))$ determines candidate critical residual directions and their stability.
\textbf{(b)} The state splits as $x=ru$. Only the tangential part $g=P_uF$ of the field turns $u$. The radial part rescales the magnitude $r$, which sets angular speed but not the critical directions.}
\label{fig:concept-directions}
\end{figure*}
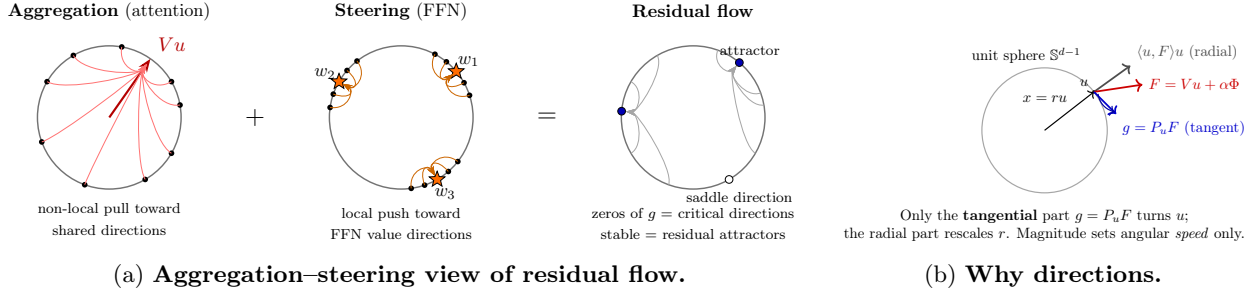

In our paper we provide the following contributions: (i) We formulate residual dynamics as an aggregation--steering flow on the sphere, with attention as the non-local aggregation term and the FFN as a local steering term. (ii) We identify the tangential FFN field $P_u\Phi(u)$, nonlinear equilibria, tangent stability, FFN spectral shifts, and the attention--FFN commutator as mathematical objects. (iii) We validate the framework with four experiments: (1) an attention-only angular-deficit diagnostic with an explicit scale trend, (2) tangential/radial FFN ablation, (3) diversity under aggregation pressure, and (4) practical sequential-to-parallel defect-guided parallelization. Additional diagnostics -- FFN-gate calibration, prefix/Volterra sensitivity, local spectrum geometry of approximate critical candidates, OV non-symmetry, and editing locality -- are provided in the Appendix.

\section{Problem Setting}

Let $x_i^\ell=r_i^\ell u_i^\ell$, where $u_i^\ell\in\Sph^{d-1}$. For a continuous-depth idealization $\dot x_i=F_i$,
\begin{equation}
\dot r_i=\langle u_i,F_i\rangle,\quad
\dot u_i=\frac1{r_i}P_{u_i}F_i,\quad P_u=I-uu^\top,
\label{eq:split}
\end{equation}
where $P_u$ is the orthogonal projection onto the tangent space at $u$. Only the tangential component changes direction, and radial motion changes norm and therefore angular speed. Attention contributes a non-local aggregation field
\begin{equation}
A_i(U)=\sum_{j\in\mathcal N(i)}a_{ij}(U)Vu_j .
\end{equation}
For exposition one may write a standard FFN as
\begin{equation}
\Phi(u)=\sum_p\gamma_p w_p\sigma(\langle k_p,u\rangle),
\label{eq:phi}
\end{equation}
but the experiments use the actual checkpoint FFNs. For Llama/Mistral-style SwiGLU blocks \citep{shazeer2020glu,jiang2023mistral,dubey2024llama} this local steering has the form
\begin{equation}
\Phi(u)=W_{\rm down}\!\left[\sigma(W_{\rm gate}u)\odot(W_{\rm up}u)\right],
\end{equation}
which is still a Lipschitz position-wise field and therefore fits the same dynamical
framework. Thus the additive key--value notation in Eq.~\eqref{eq:phi} is a
mathematical proxy, not a claim that gated FFNs have literal key-equals-value rows;
our architecture-robust results use the actual checkpoint MLP/SwiGLU modules and
measure the induced direction field.

The full model is the particle system in Eq.~\eqref{eq:flow}. The simplified
single-particle field
\begin{equation}
g(u)=P_u\big(Vu+\alpha \Phi(u)\big)
\label{eq:single}
\end{equation}
is used only as a local phase-portrait approximation near a coherent token cluster and not a global model of a causal decoder. Zeros of this projected field are \emph{critical residual directions}. Stable zeros are \emph{residual attractors}, and mixed-stability zeros are \emph{saddle residual directions}.

% \begin{figure}[!t]\centering
% \begin{tikzpicture}[scale=0.95, every node/.style={font=\small}]
%   \draw[gray!70, thick] (0,0) circle (1.55);
%   \node at (-0.45,1.86) {unit sphere $\Sph^{d-1}$};

%   \coordinate (O) at (0,0);
%   \coordinate (X) at (38:1.55);
%   \coordinate (R) at ($(X)+(0.90,0.65)$);
%   \coordinate (F) at ($(X)+(1.20,0.18)$);
%   \coordinate (T) at ($(X)+(0.55,-0.55)$);

%   \draw[->, thick] (O) -- (X) node[pos=0.58, above left] {$x=ru$};
%   \fill (X) circle (1.2pt) node[above left] {$u$};

%   \draw[->, very thick, black!70] (X) -- (R)
%     node[above right] {$\langle u,F\rangle u$ (radial)};
%   \draw[->, very thick, red!80!black] (X) -- (F)
%     node[right] {$F=Vu+\alpha\Phi$};
%   \draw[->, very thick, blue!80!black] (X) -- (T)
%     node[below right] {$g=P_uF$ (tangent)};

%   \draw[->, blue!65!black, thick] ($(X)+(0.10,-0.05)$)
%     .. controls ($(X)+(0.15,-0.35)$) and ($(X)+(0.35,-0.45)$) .. ($(X)+(0.52,-0.50)$);

%   \node[align=center, font=\small] at (0,-2.35)
%   {Only the \textbf{tangential} part $g=P_uF$ turns $u$;\\
%    the radial part rescales $r$. Magnitude sets angular \emph{speed} only.};
% \end{tikzpicture}
% \caption{\textbf{Why directions.} The state splits as $x=r u$. Only the
% tangential part $g=P_uF$ of the field turns $u$; the radial part rescales the
% magnitude $r$, which sets angular speed but not the critical directions.}
% \label{fig:directions}
% \end{figure}

For full attention, the usual exchangeable mean-field limit applies. Decoder-only models are causal and instead induce a prefix-indexed Volterra flow: if $s=i/L$ and $\nu_{t,s}=s^{-1}\int_0^s\delta_{u(t,r)}dr$, then
in normalized time the continuum velocity is $P_u[A[\nu_{t,s}](u)+\Phi(u)]$. We treat the mean-field
theorem as the clean bidirectional case and test the causal prefix object empirically
in the appendix.

\section{Theory}

This section presents the mathematical objects used by the experiments and full proofs are given in the appendix. The construction follows the continuous attention-dynamics viewpoint of prior work \citep{geshkovski2023mathematical,karagodin2024causal,karagodin2025normalization}, but adds the FFN as a local steering field.

\begin{theorem}[Well-posedness; short form]\label{thm:wp}
Assume every attention mask is fixed and nonempty, the softmax temperature satisfies
$\tau>0$, each local field $\Phi_i:\Sph^{d-1}\to\R^d$ is Lipschitz, and the
prescribed speed factors are measurable with $s_i(t)\ge s_->0$ almost everywhere.
Then, for every initial condition in $(\Sph^{d-1})^L$, Eq.~\eqref{eq:flow} has a
unique global absolutely continuous solution that remains in
$(\Sph^{d-1})^L$. If the coefficients are continuous in time, the solution is
classical, and if the system is autonomous, these solutions define a global flow.
ReLU is covered because it is globally Lipschitz.
\end{theorem}

This theorem ensures that the proposed particle dynamics is a well-defined ODE on the product of spheres. 

\paragraph{Equilibria and stability.}
Next we analyze critical points of the projected system. A critical residual direction $u_\star$ satisfies
\begin{equation}
Vu_\star+\alpha\Phi(u_\star)=\rho_\star u_\star.
\label{eq:eq}
\end{equation}

Its tangent stability is controlled by
\begin{equation}
J_\star=P_{u_\star}\big(V+\alpha D\Phi(u_\star)-\rho_\star I\big)|_{T_{u_\star}\Sph^{d-1}}.
\label{eq:jac}
\end{equation}

The corresponding stability criterion is spectral:
\begin{proposition}[Stability test]\label{prop:stab}
Assume that $\Phi$ is $C^1$ in a neighborhood of $u_\star$. Then
Eq.~\eqref{eq:jac} is the intrinsic tangent linearization of $g_\alpha$ at
$u_\star$. If $u_\star$ is hyperbolic, it is a locally exponentially
asymptotically stable residual attractor iff $\max\Re\Spec(J_\star)<0$.
\end{proposition}

\paragraph{How FFN spectrum changes clustering.}
At a fixed critical direction $u_\star$ for a fixed value of $\alpha$, set
$\rho_A=\langle u_\star,Vu_\star\rangle$ and
$\rho_\Phi=\langle u_\star,\Phi(u_\star)\rangle$. On
$T=T_{u_\star}\Sph^{d-1}$ define
$J_A(u_\star)=P_{u_\star}(V-\rho_A I)|_T$ and
$J_\Phi(u_\star)=P_{u_\star}(D\Phi(u_\star)-\rho_\Phi I)|_T$. Then the exact
pointwise decomposition is
\begin{equation}
J_\star=J_A(u_\star)+\alpha J_\Phi(u_\star).
\label{eq:ffnspec}
\end{equation}
Along a moving equilibrium branch $u_\alpha$, both operators in
Eq.~\eqref{eq:ffnspec} generally depend on $\alpha$, so no affine eigenvalue law
follows. If instead $u_\star$ is a common critical direction,
$Vu_\star=\rho_Au_\star$ and $\Phi(u_\star)=\rho_\Phi u_\star$, and
$J_A(u_\star),J_\Phi(u_\star)$ are simultaneously diagonalizable over
$\mathbb C$, their common modes satisfy
$\lambda_k(\alpha)=\lambda_k^A+\alpha\lambda_k^\Phi$.

\paragraph{Steering, invariant regions, and order defects.}
Assume that $\Phi$ is $C^2$ near a hyperbolic attention-only critical direction $u_0$, and let $J_0$ denote its tangent Jacobian. The first-order expansion established in Theorem~\ref{thm:steer} below is
\begin{equation}
u_\alpha=u_0-\alpha J_0^{-1}P_{u_0}\Phi(u_0)+O(\alpha^2),
\label{eq:steer}
\end{equation}
which motivates FFN-value editing. If two separated regions carry masses $p$ and
$1-p$ and every cross-region angle is at least $\theta_0$, then
\begin{equation}D=1-\|\mathbb E u\|^2\ge 2p(1-p)(1-\cos\theta_0).
\label{eq:anti}\end{equation}
\begin{corollary}[Conditional anti-collapse; short form]\label{cor:anti}
Let $S_1,S_2$ be disjoint Borel sets, let $0\le\theta_0\le\pi$, and
let the law of $u_t$ be supported on $S_1\cup S_2$, with masses $p$ and $1-p$.
Assume $\langle v,w\rangle\le\cos\theta_0$ for every $v\in S_1$,
$w\in S_2$. Then Eq.~\eqref{eq:anti} holds. If $S_1$ and $S_2$
are forward-invariant and the initial law has these masses, the same lower bound
holds for every later time.
\end{corollary}
For the exact flows of smooth vector fields, the Lie--Trotter expansion is
\begin{equation}
\begin{aligned}
\varphi^\Phi_\tau\circ\varphi^A_\tau
&=\varphi^{A+\Phi}_\tau
  +\frac{\tau^2}{2}[A,\Phi]+O(\tau^3),\\[-1mm]
[A,\Phi]&=D\Phi\,A-DA\,\Phi.
\end{aligned}
\label{eq:comm}
\end{equation}
Thus the Lie bracket controls exact-flow splitting and order reversal. The direct
defect of the discrete sequential-to-parallel residual-map replacement is a
different directional-derivative term, stated in Proposition~\ref{prop:comm}.
\begin{lemma}[Invariant-region and equilibrium criterion; short form]\label{lem:basin}
Assume $d\ge2$, the sum in Eq.~\eqref{eq:phi} is finite, the value directions $w_q$ are
unit vectors, $\sigma$ is Lipschitz, and $\alpha>0$. Fix $p$ and
$0<\Delta<\pi/2$, and let
$B_p=\{u\in\Sph^{d-1}:\angle(u,w_p)\le\Delta\}$. Define
\[
\begin{aligned}
m_p
&=\inf_{u\in\partial B_p}
  \gamma_p\sigma(\langle k_p,u\rangle),\\
C_{pq}
&=\sup_{u\in\partial B_p}
  \bigl|\gamma_q\sigma(\langle k_q,u\rangle)\bigr|,\\
M_p
&=\sup_{u\in\partial B_p}\|P_uVu\|.
\end{aligned}
\]
If $m_p>0$ and
\begin{equation}
\alpha m_p\sin\Delta>M_p+\alpha\sum_{q\ne p}C_{pq},
\label{eq:inward}
\end{equation}
then $B_p$ is forward-invariant for $g_\alpha$ and contains at least one critical
residual direction in its interior. Forward invariance alone does not imply
attraction; any hyperbolic critical direction is classified by
Proposition~\ref{prop:stab}.
\end{lemma}
The invariant-region condition is an existence statement, not a typical-regime
claim. It describes a separated, sufficiently strong-gating corner in which FFN
steering can isolate regions of state space. The empirical saddle study in the
appendix maps the more common operating regime, where the FFN often
preserves diversity through saddle/rotational transport rather than by creating many
residual attractors.
\begin{theorem}[Mean-field limit; short form]\label{thm:mf}
Consider normalized-time ($s_i\equiv1$), homogeneous all-to-all full attention
(including self-interaction) with shared
$Q,K,V,\Phi$, temperature $\tau>0$, and Lipschitz $\Phi$. Let
$\mu_t^L=L^{-1}\sum_{i=1}^L\delta_{u_i^L(t)}$. For every initial law
$\bar\mu_0$ there is a unique global solution of
\begin{equation}
\dot{\bar u}_t=P_{\bar u_t}\big(A[\bar\mu_t](\bar u_t)+\alpha\Phi(\bar u_t)\big),
\qquad \bar\mu_t=\mathrm{Law}(\bar u_t),
\label{eq:mckean}
\end{equation}
and, for every $T<\infty$, a constant $C_T$ independent of $L$ such that
\[
\sup_{0\le t\le T}W_1(\mu_t^L,\bar\mu_t)
\le C_T W_1(\mu_0^L,\bar\mu_0).
\]
Consequently, $W_1(\mu_0^L,\bar\mu_0)\to0$ implies uniform convergence on
compact time intervals. In particular, the conclusion holds almost surely for
i.i.d.\ initial particles with law $\bar\mu_0$.
\end{theorem}
\begin{proposition}[Exact-flow commutator and residual-map defect]\label{prop:comm}
Let $A$ and $\Phi$ be $C^2$ vector fields on a neighborhood of the states under
consideration. For their exact flows, Eq.~\eqref{eq:comm} holds locally uniformly,
and
\[
\varphi^\Phi_\tau\circ\varphi^A_\tau-
\varphi^A_\tau\circ\varphi^\Phi_\tau
=\tau^2[A,\Phi]+O(\tau^3).
\]
For the residual maps
\[
\begin{aligned}
S_\tau(x)&=x+\tau A(x)+\tau\Phi(x+\tau A(x)),\\
P_\tau(x)&=x+\tau A(x)+\tau\Phi(x),\\
R_\tau(x)&=x+\tau\Phi(x)+\tau A(x+\tau\Phi(x)),
\end{aligned}
\]
one has
\[
\begin{aligned}
(S_\tau-P_\tau)(x)&=\tau^2D\Phi(x)A(x)+O(\tau^3),\\
(R_\tau-P_\tau)(x)&=\tau^2DA(x)\Phi(x)+O(\tau^3),\\
(S_\tau-R_\tau)(x)&=\tau^2[A,\Phi](x)+O(\tau^3).
\end{aligned}
\]
Thus the Lie bracket controls order reversal, whereas the direct
sequential-to-parallel defect is governed to leading order by $D\Phi\,A$ and is
not determined by the bracket alone.
\end{proposition}
\begin{theorem}[First-order steering]\label{thm:steer}
Assume that $\Phi$ is $C^2$ near a hyperbolic critical direction $u_0$ of the
attention-only field $g_0(u)=P_uVu$, and let $J_0$ be its tangent Jacobian. Then
there are $\alpha_0>0$ and a unique $C^2$ branch of critical directions
$u_\alpha$ near $u_0$ for $|\alpha|<\alpha_0$, and the expansion in Eq.~\eqref{eq:steer}
holds in ambient Euclidean norm. If $u_0$ is a residual attractor, then
$u_\alpha$ remains a residual attractor for all sufficiently small $|\alpha|$.
\end{theorem}

\section{Experiments}
\label{sec:exp}
\paragraph{Protocol.} We evaluate GPT-2 and GPT-2-large \citep{radford2019language}, Pythia-410M and Pythia-1.4B \citep{biderman2023pythia}, Mistral-7B \citep{jiang2023mistral}, and Llama-3-8B \citep{dubey2024llama} on OpenWebText \citep{gokaslan2019openwebtext}. The main text reports four experiment groups: (i) an attention-only angular-deficit diagnostic, (ii) radial/tangential FFN ablations, (iii) diversity under aggregation pressure, and (iv) sequential-to-parallel defect-guided approximate parallelization. Additional diagnostics and proofs are presented in the appendix.
\paragraph{Attention-only angular deficit.}
We ask whether the attention output alone aligns with the realized angular update.
For each token and layer, write $x_i^\ell=r_i^\ell u_i^\ell$, with
$u_i^\ell=x_i^\ell/\|x_i^\ell\|$, and let $A_i^\ell$ and $\Phi_i^\ell$
denote the realized attention and FFN branch outputs, so that
$x_i^{\ell+1}=x_i^\ell+A_i^\ell+\Phi_i^\ell$. The observed tangent
displacement is
\begin{equation}
\Delta u_i^\ell
=
P_{u_i^\ell}\left(u_i^{\ell+1}-u_i^\ell\right),
\end{equation}
and the attention-only tangent direction is
\begin{equation}
\widehat v_i^{\,A}=P_{u_i^\ell}A_i^\ell.
\end{equation}
Because $P_{u_i^\ell}x_i^\ell=0$, the realized full branch sum satisfies
\begin{equation}
\Delta u_i^\ell
=
\frac{1}{\|x_i^{\ell+1}\|}
P_{u_i^\ell}\left(A_i^\ell+\Phi_i^\ell\right).
\end{equation}
Thus the full-update tangent direction is algebraically collinear with the
observed displacement and is not an independent prediction. We therefore report
the attention-only angular deficit
\begin{equation}
\label{eq:att_def}
D_A
=
1-
\cos\left(\widehat v_i^{\,A},\Delta u_i^\ell\right),
\end{equation}
which is numerically identical to the previously written alignment difference,
averaged over tokens, layers, and texts.

The deficit is positive in all tested models (Fig. \ref{fig:onestep}) and ranges from $0.41$ for GPT-2 to
$0.63$ for Llama-3-8B. Hence attention alone becomes increasingly misaligned
with the realized block update in the tested scale range. This is an attribution
diagnostic rather than an independent validation of the complete field, but it
is consistent with the FFN contributing directly to the realized angular update
in the aggregation--steering model.
% \begin{equation}
% \dot u_i
% =
% \frac{1}{s_i(t)}P_{u_i}
% \left(A_i(U)+\Phi_i(u_i)\right).
% \end{equation}

\begin{figure}[H]\centering
\includegraphics[width=0.95\linewidth]{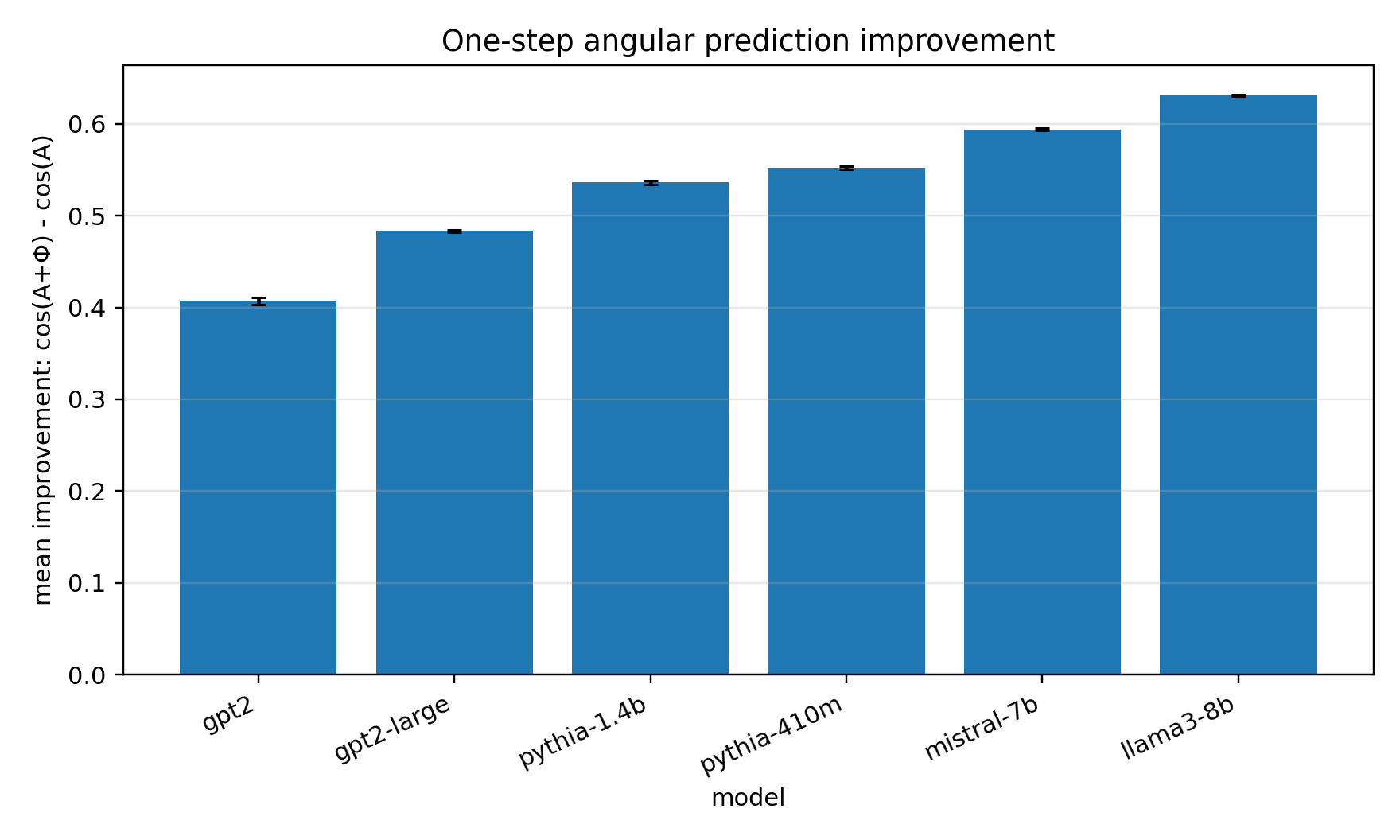}
\caption{\textbf{Attention-only angular deficit is positive in all tested models.} The plotted quantity is defined by \eqref{eq:att_def}.}
\label{fig:onestep}
\end{figure}
\FloatBarrier

\paragraph{Tangential effect of FFN.}
The angular decomposition in Eq.~\eqref{eq:split} predicts that only the
tangential component of the FFN output can change the residual direction.
Indeed, for a residual direction $u$, we can decompose the FFN update as
\begin{equation}
\Phi(u)
=
\underbrace{
\left\langle u,\Phi(u)\right\rangle u
}_{\text{radial component}}
+
\underbrace{
P_u\Phi(u)
}_{\text{tangential component}}
\end{equation}
%The radial part changes the norm scale of the residual stream, while the tangential part changes the direction on the sphere. 
In the finite-block
experiments, $u$ denotes the normalized residual state immediately before the FFN
output is added. A finite radial
update preserves direction only when its total scalar coefficient does not cross
zero. This leads to a direct ablation test: if FFN is useful mainly through
angular steering, then preserving only $P_u\Phi(u)$ should retain most of the
model quality, whereas preserving only the radial component should behave
similarly to removing full FFN.

We therefore compare five forward-pass variants. The \emph{full} model uses the
original block update. The \emph{zero-FFN} variant removes the FFN output. The
\emph{radial-only} variant keeps only
$\left\langle u,\Phi(u)\right\rangle u$. The \emph{tangential-only} variant keeps
only $P_u\Phi(u)$. Finally, the \emph{attention-zero} variant removes the
attention output as a separate control. We evaluate each variant by the increase
in language-modeling loss relative to the full model.

Across all tested models (Fig. \ref{fig:ablation}), the tangential-only FFN is consistently the closest
to the full model. In contrast,
the radial-only FFN is catastrophic comparable to removing the FFN entirely. Thus the useful part of the
FFN update is primarily tangential steering of residual directions.

\begin{figure}[H]\centering
\includegraphics[width=0.8\linewidth]{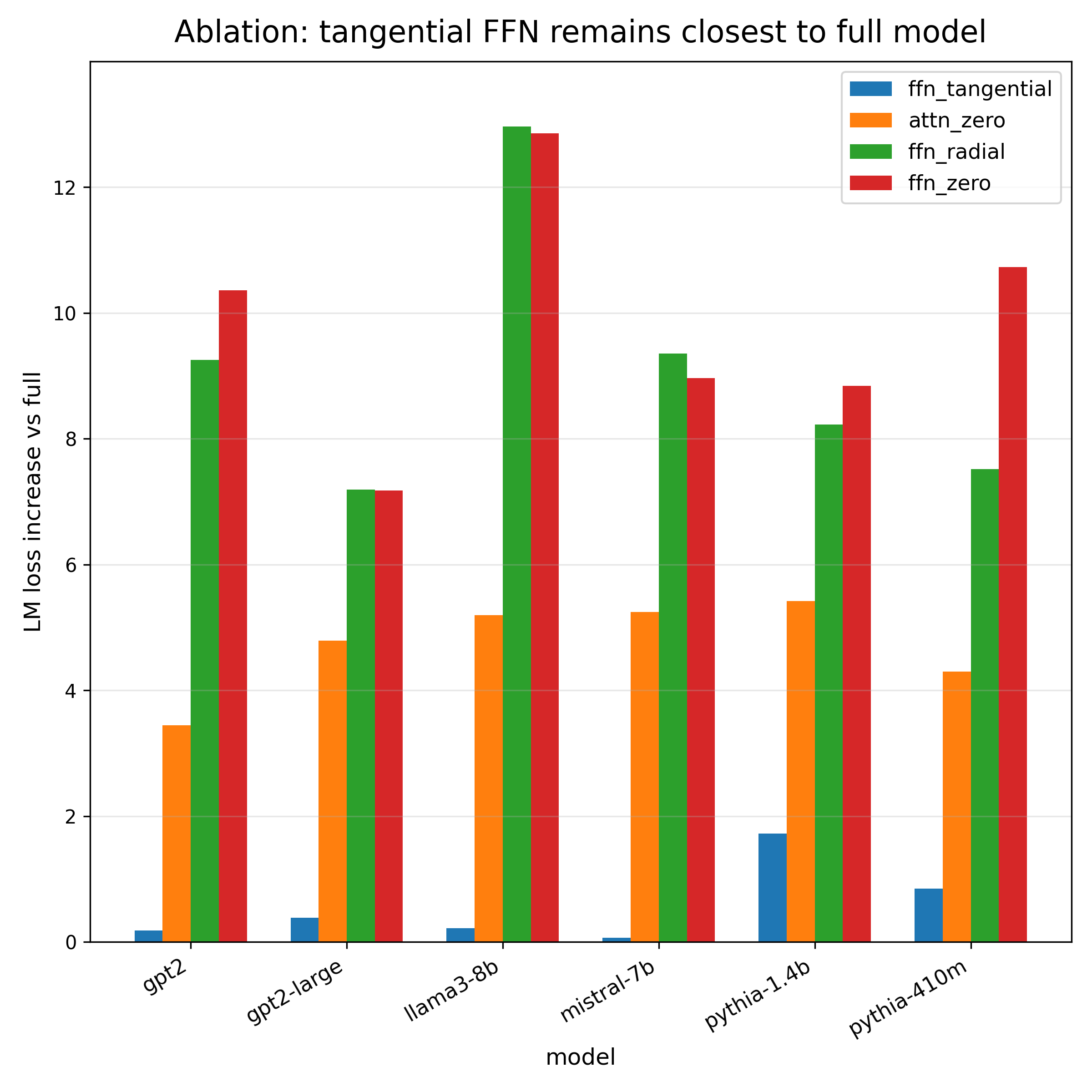}
\caption{\textbf{Tangential FFN removal has similar effect to full FFN removal.} Loss increase relative
to the full model under component ablations.}
\label{fig:ablation}
\end{figure}
\FloatBarrier

% \begin{figure*}[t]\centering
% \includegraphics[width=0.8\linewidth]{figures/res_diversity_collapse.png}
% \caption{\textbf{Attention contracts and FFN often restores diversity.} Left:
% change in spherical diversity (eq. \eqref{eq:sphere_div}). Right: under
% amplified attention, tangential FFN preserves effective rank far better than
% radial-only or zero-FFN variants. This is consistent with the saddle study: FFN
% preserves diversity by adding transport/rotational modes, not merely by creating new
% attractors.}
% \label{fig:divcollapse}
% \end{figure*}

\paragraph{FFN controls diversity under aggregation pressure.}
We next test whether the tangential FFN component helps prevent the residual
directions from over-concentrating under attention aggregation. For a collection
of residual directions with empirical measure $\mu$, we use the spherical
diversity statistic
\begin{equation}
D(\mu)
=
1-\left\|
\mathbb{E}_{u\sim\mu} u
\right\|^2.
\label{eq:sphere_div}
\end{equation}
This quantity is small when many directions concentrate around a common mean
direction, and larger when the directions remain spread out on the sphere. We
measure $D(\mu)$ at three points in each block: before attention, after
attention, and after the FFN (fig. \ref{fig:divcollapse}). This lets us separate the effect of aggregation
from the effect of the local FFN steering.

Empirically, the attention step often decreases
$D(\mu)$, indicating that aggregation pulls residual directions toward a more
concentrated configuration. On the other hand, in many cases the FFN increases the diversity after attention, suggesting that it acts against excessive
aggregation in the angular variables.

% To make this effect causal rather than purely descriptive, we introduce an
% over-aggregation intervention. For a sequential block, we amplify the attention
% output by a factor $\beta_A=2$ and use
% \begin{equation}
% y_\ell^{(\beta)}=x_\ell+\beta_A A_\ell(x_\ell),
% x_{\ell+1}=y_\ell^{(\beta)}+\Phi_\ell\!\left(y_\ell^{(\beta)}\right).
% \end{equation}
% The tangential/radial ablations decompose the realized FFN output relative to
% $y_\ell^{(\beta)}/\|y_\ell^{(\beta)}\|$. For Pythia-style parallel-residual
% blocks, both branches read $x_\ell$, so the native intervention is
% $x_{\ell+1}=x_\ell+\beta_A A_\ell(x_\ell)+\Phi_\ell(x_\ell)$ and is reported
% separately. We then measure not only loss, but also the effective rank of the
% residual representation. If the tangential FFN component is responsible for
% preserving representation geometry, then the tangential-only variant should
% retain substantially more effective rank than the radial-only or zero-FFN
% variants under the same amplified attention pressure.

\begin{figure}[H]\centering
\includegraphics[width=0.9\linewidth]{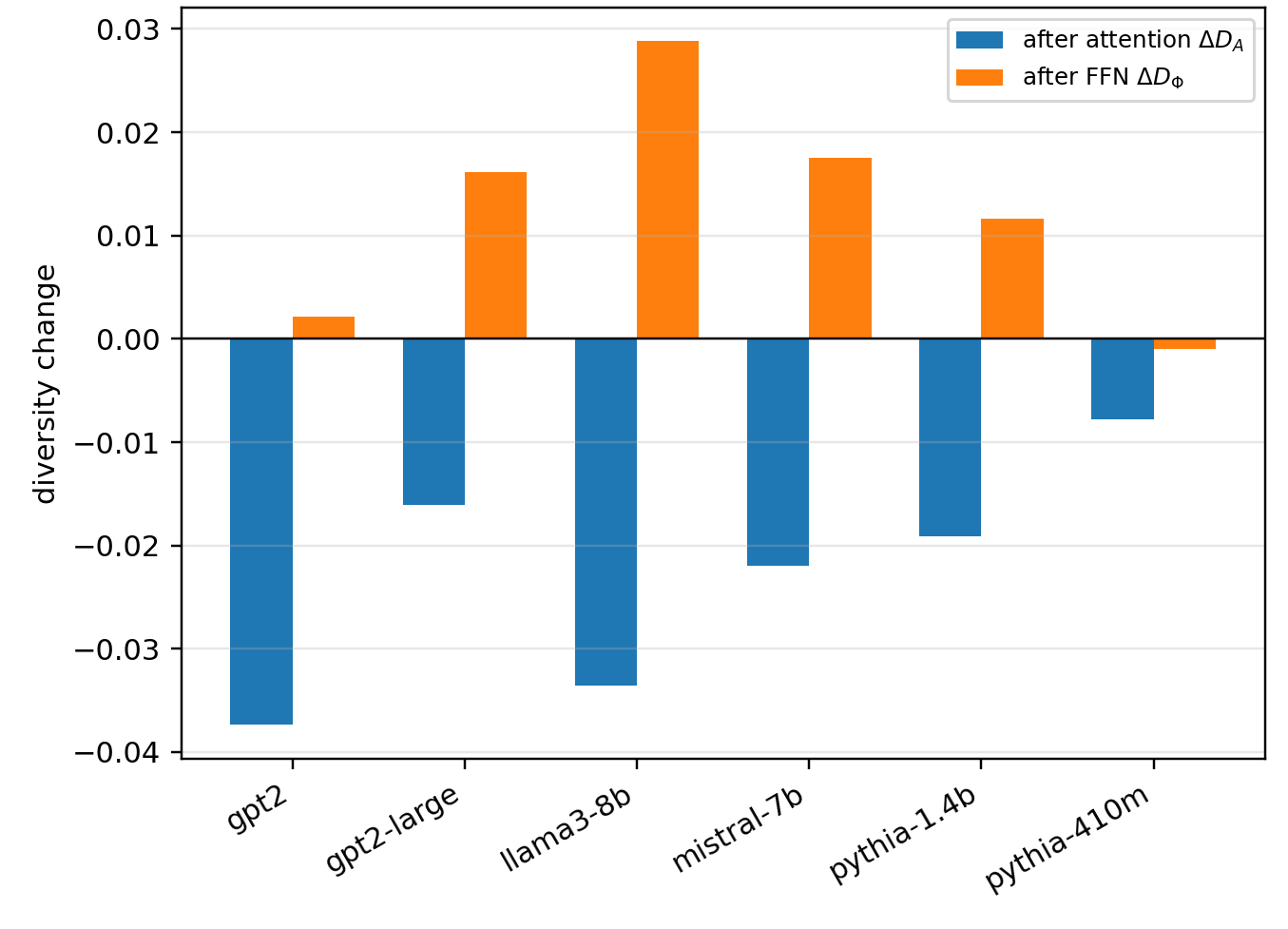}
\caption{\textbf{Attention contracts and FFN often restores diversity.} 
Change in spherical diversity (eq. \eqref{eq:sphere_div}) after attention and FFN for different models.}
\label{fig:divcollapse}
\end{figure}

Together with the saddle-geometry appendix results (see Table 4), the interpretation is not that FFN
simply creates more stable attractors. Rather, the tangential FFN component
injects non-collapsing saddle-like and rotational transport modes. These modes
keep the residual directions from concentrating too aggressively, allowing the
model to maintain a higher-dimensional representation geometry even when
attention applies strong aggregation pressure.

%\FloatBarrier
\paragraph{Sequential-to-parallel defect-guided parallelization.}
The residual-map part of Proposition~\ref{prop:comm} suggests a practical
engineering test. In a standard sequential Transformer block, the FFN does not
read the same state as attention. Instead, attention is applied first, and the FFN
is evaluated on the attention-updated residual stream:
\begin{equation}
x
\longmapsto
x+A(x)+\Phi(x+A(x)).
\end{equation}
If the FFN output is only weakly sensitive to the attention-induced change of its
input, then replacing this sequential computation by a same-input parallel
surrogate should introduce only a small error:
\begin{equation}
x
\longmapsto
x+A(x)+\Phi(x).
\end{equation}
In this surrogate, the attention branch and the FFN branch read the same input
$x$ and can in principle be evaluated concurrently.

\begin{figure*}[!t]
\centering
\includegraphics[width=0.9\linewidth]{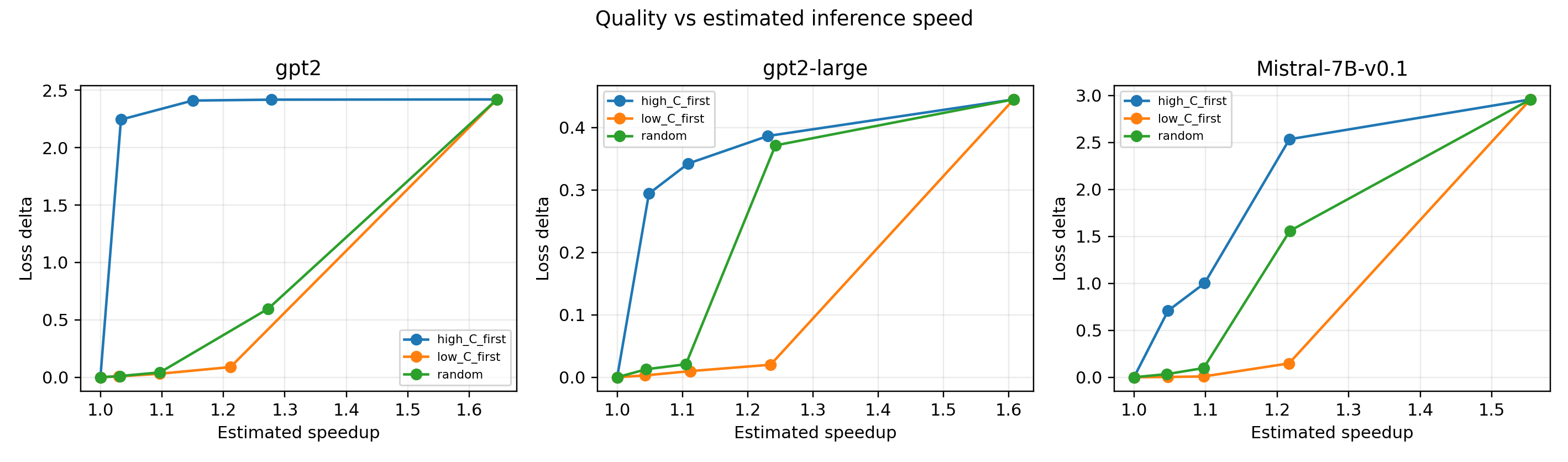}
\caption{\textbf{Practical sequential-to-parallel defect-guided parallelization.}
We replace selected sequential attention-FFN blocks
$x+A(x)+\Phi(x+A(x))$ by the same-input parallel surrogate
$x+A(x)+\Phi(x)$. Layers are selected by the measured FFN input-sensitivity
score $C_\ell$. Low-defect layers preserve quality much better than high-defect
layers at the same estimated submodule-time reduction. Each point represent the corresponding relative change of layers in $\{0\%, 25\%, 50\%, 75\%, 100\%\}$.}
\label{fig:practicalparallel}
\end{figure*}

% The relevant question is therefore not whether every layer can be parallelized,
% but whether the direct sequential-to-parallel defect identifies the layers where
% this replacement is least harmful. 
For each layer $\ell$, we estimate an FFN
input-sensitivity score $C_\ell$ by measuring how much the FFN output changes
when its sequential post-attention input is replaced by the pre-attention input.
Concretely, we use a normalized proxy of the form
\begin{equation}
C_\ell
\approx
\frac{
\left\|
\Phi_\ell\left(x+A_\ell(x)\right)-\Phi_\ell(x)
\right\|
}{
\left\|
\Phi_\ell(x)
\right\|+\varepsilon
}.
\end{equation}
Small $C_\ell$ means that the FFN is relatively insensitive to whether attention
has already been applied, so the parallel surrogate should be accurate. Large
$C_\ell$ means that the attention--FFN order is itself an important part of the
computation.

We test this by ranking layers according to $C_\ell$ and replacing selected
sequential blocks by the parallel surrogate. We compare three schedules:
low-$C_\ell$-first, random, and high-$C_\ell$-first. For each schedule, we
measure the language-modeling loss after replacing an increasing number of
layers. We also report an estimated inference-time reduction based on replacing
the sequential layer cost $T_A^\ell+T_\Phi^\ell$ by the idealized parallel cost
\begin{equation}
T_{\mathrm{parallel}}^\ell
=
\max\left(T_A^\ell,T_\Phi^\ell\right),
\end{equation}
so that the estimated saving in a parallelized layer is
\begin{equation}
\Delta T^\ell
=
\min\left(T_A^\ell,T_\Phi^\ell\right).
\end{equation}

This speedup is therefore a submodule-time estimate rather than a claim about an
optimized production kernel.

The results (Fig. \ref{fig:practicalparallel}) show that the direct defect score is predictive of practical
replaceability. The low-$C_\ell$ strategy preserves quality substantially better
than random or high-$C_\ell$ selection in GPT-2, GPT-2-large, and Mistral. For
GPT-2-large, parallelizing $18$ low-$C_\ell$ layers gives an estimated
$1.24\times$ speedup with only $+0.02$ loss, while choosing high-$C_\ell$ layers
already incurs $+0.29$ loss after only four layers. For Mistral, parallelizing
eight low-$C_\ell$ layers gives an estimated $1.10\times$ speedup with
$+0.009$ loss, whereas four high-$C_\ell$ layers give $+0.71$ loss. Thus
$C_\ell$ is not only an ODE-validity diagnostic but an actionable
layer-selection quantity.

% Pythia is treated separately. In Pythia-style parallel-residual architectures,
% the attention and FFN submodules already read the same residual input. Therefore
% the sequential-versus-parallel replacement is much less meaningful as an
% intervention: a flatter reordering curve is the expected architectural behavior,
% not a counterexample to the sequential-composition interpretation. This supports
% the interpretation that $C_\ell$ measures FFN input sensitivity specifically in
% architectures where attention and FFN are composed sequentially.  

%\FloatBarrier
\paragraph{FFN gate diagnostics.}
As a complementary sensitivity diagnostic, we freeze the pretrained checkpoint
and learn only scalar layerwise FFN gates. For a sequential block, the update is
\begin{equation}
y_\ell=x_\ell+A_\ell(x_\ell),
\qquad
x_{\ell+1}=y_\ell+\alpha_\ell\Phi_\ell(y_\ell).
\end{equation}
For Pythia-style parallel-residual blocks, both gated and attention branches read
$x_\ell$, following the native architecture. All model weights are fixed and only
the parameters
$\{\alpha_\ell\}_{\ell=1}^L$ are optimized on a small calibration set. The
natural pretrained model corresponds to the baseline $\alpha_\ell=1$ for all
layers. Thus this experiment asks whether the pretrained FFN steering scale is
locally over-used, under-used, or already close to a good operating point.

The comparison is made against the natural checkpoint. We report the learned
mean gate, the range of learned gates across layers, and the loss difference
\begin{equation}
\Delta L
=
L_{\mathrm{learned\ gates}}
-
L_{\alpha_\ell=1}.
\end{equation}
Negative $\Delta L$ means that the learned gates slightly improve over the
natural FFN scale, while positive $\Delta L$ means that the learned gate profile
does not improve the frozen checkpoint.

\begin{table}[H]
\centering
\small
\caption{\textbf{Layerwise FFN gate diagnostics.}
We freeze all checkpoint weights and learn only scalar FFN gates
$\alpha_\ell$ in the native block: sequential models use
$y_\ell=x_\ell+A_\ell(x_\ell)$ and
$x_{\ell+1}=y_\ell+\alpha_\ell\Phi_\ell(y_\ell)$, while Pythia retains its
parallel-residual input convention. The table reports the mean learned gate, its
layerwise range, and the loss
difference relative to the natural baseline $\alpha_\ell=1$. Negative
$\Delta L$ indicates a small improvement over the natural checkpoint.}
\label{tab:ffn-gate-diagnostics}
\begin{tabular}{lccc}
\toprule
Model & Mean $\alpha_\ell$ & Range of $\alpha_\ell$ & $\Delta L$ vs. $\alpha_\ell=1$ \\
\midrule
GPT-2        & $0.852$ & $0.667$--$0.950$ & $-0.028$ \\
GPT-2-large  & $0.917$ & $0.655$--$1.119$ & $-0.005$ \\
Pythia-410M  & $1.063$ & $0.709$--$1.412$ & $+0.077$ \\
Pythia-1.4B  & $1.048$ & $0.760$--$1.207$ & $+0.017$ \\
Mistral-7B   & $1.009$ & $0.861$--$1.098$ & $+0.006$ \\
\bottomrule
\end{tabular}
\end{table}

% Place the wide table early enough for top-of-page placement in two-column mode.
\begin{table*}[!t]
\centering
\small
\caption{\textbf{Component ablation of approximate critical-candidate geometry.}
We compare Attention-only $P_uVu$, FFN-only $P_u\Phi(u)$, and the full field
$P_u(Vu+\Phi(u))$. Counts are reported at $\alpha=1$ with zero-score threshold
$\|g(u_\star)\|/(\|F(u_\star)\|+\varepsilon)\leq 0.05$. Each entry reports
A-like/S-like/R-like, the number of locally contracting, mixed-sign, and locally
expanding tangent-spectrum candidates, followed by the S-like fraction in
parentheses. These are descriptive local-linearization labels, not equilibrium
stability classifications.}
\label{tab:ffn-restructure-critical-geometry-wide}
\begin{tabular}{lrrrr}
\toprule
Model
& Total
& Attention-only
& FFN-only
& Full \\
\midrule
GPT-2
& $34/21/21$
& $6/25/3$ \;($73.5\%$)
& $3/14/4$ \;($66.7\%$)
& $3/13/5$ \;($61.9\%$) \\
GPT-2-large
& $22/37/36$
& $2/17/3$ \;($77.3\%$)
& $2/32/3$ \;($86.5\%$)
& $2/30/4$ \;($83.3\%$) \\
Pythia-410M
& $26/36/36$
& $2/21/3$ \;($80.8\%$)
& $1/30/5$ \;($83.3\%$)
& $1/29/6$ \;($80.6\%$) \\
Pythia-1.4B
& $18/46/47$
& $5/10/3$ \;($55.6\%$)
& $2/42/2$ \;($91.3\%$)
& $2/44/1$ \;($93.6\%$) \\
\midrule
All models
& $100/140/140$
& $15/73/12$ \;($73.0\%$)
& $8/118/14$ \;($84.3\%$)
& $8/116/16$ \;($82.9\%$) \\
\bottomrule
\end{tabular}
\end{table*}

The learned gates are nontrivial and architecture-dependent. GPT-2 and
GPT-2-large reduce the average FFN gain, with mean gates $0.852$ and $0.917$,
respectively. In contrast, the Pythia models slightly increase the average FFN
gain, with mean gates $1.063$ and $1.048$. Mistral remains very close to the
natural scale, with mean gate $1.009$. The loss changes are small and not
universal: GPT-2 obtains a modest improvement, GPT-2-large improves only
slightly, while the Pythia and Mistral runs do not improve over the natural
baseline.

We therefore interpret the gate profile as a cheap diagnostic of layerwise FFN
scale sensitivity rather than as a competitive fine-tuning method. The fact that
the learned gates remain close to one, and that improvements are small, suggests
that pretraining already calibrates the FFN steering scale near a good operating
point. At the same time, the non-flat and architecture-dependent gate profiles
show that the steering strength is not arbitrary: different model families use
different layerwise balances between aggregation and FFN steering.

\paragraph{FFN reshapes the local critical geometry.}
To isolate how the FFN reorganizes residual directions, we perform a component
ablation of the reduced critical-direction field. We compare three fields:
\begin{equation}
g_{\mathrm{OV}}(u)
=
P_uVu,
\end{equation}
\begin{equation}
g_{\mathrm{FFN}}(u)
=
P_u\Phi(u),
\end{equation}
and
\begin{equation}
g_{\mathrm{full}}(u)
=
P_u\left(Vu+\Phi(u)\right).
\end{equation}
For each field, we search for candidate critical residual directions
$u_\star$ satisfying $g(u_\star)\approx 0$, using the normalized zero-score
criterion
\begin{equation}
\frac{\|g(u_\star)\|}{\|F(u_\star)\|+\varepsilon}
\leq 0.05.
\end{equation}
At each approximate candidate, we compute the tangent-projected local Jacobian
and assign an attractor-like, saddle-like, or repeller-like spectrum label.
Because the zero-score threshold does not refine candidates to exact roots, these
are descriptive local-linearization types rather than equilibrium-stability
classifications. Neutral/rotational cases were not observed in this run.

Table~\ref{tab:ffn-restructure-critical-geometry-wide} reports both the
aggregate comparison and the model-level breakdown. The key pattern is that
the FFN-only candidate set is already strongly mixed-sign dominated: aggregated
across models, $118/140$ FFN-only candidates are saddle-like. The full candidate
set remains similarly saddle-like dominated, with $116/140$ candidates. Thus the
FFN does not merely shift candidates toward locally contracting spectra. Rather,
$P_u\Phi(u)$ reorganizes the approximate critical-candidate geometry toward
mixed-sign/transport-like local linearizations.

The model-level comparison shows that the FFN effect is not identical across
architectures. In GPT-2, the full field shifts some saddle-like candidates toward
repeller-like spectra, suggesting more locally expanding geometry. In GPT-2-large
and both Pythia models, the FFN-only and full candidate sets are more strongly
saddle-like dominated than the OV-only set. The clearest case is Pythia-1.4B:
the OV-only field has only $55.6\%$ saddle-like candidates, while the FFN-only
and full fields have $91.3\%$ and $93.6\%$, respectively. This supports the
interpretation that FFN steering actively restructures the local linearization
geometry of approximate critical candidates rather than merely supplying radial
scaling.
\section{Discussion}

The main conclusion is that Transformer residual dynamics are not well described
by attention-only aggregation. Across the main experiments, the FFN consistently
appears as a tangential steering field that is needed to account for, preserve,
and control the geometry of residual directions.

First, attention alone has a substantial angular-alignment deficit relative to
the realized block update, ranging from $0.41$ for GPT-2 to $0.63$ for
Llama-3-8B. The full realized branch sum is algebraically aligned with the
observed tangent displacement, so this result is an attribution diagnostic rather
than an independent velocity prediction. It nevertheless shows that the FFN is
not merely a secondary norm correction: it contributes directly to the realized
angular update.

Second, the useful FFN effect is tangential. In the radial/tangential ablation,
the tangential-only FFN remains closest to the full model, with loss increases
between $0.07$ and $1.72$. In contrast, the radial-only FFN is catastrophic,
with loss increases between $7.19$ and $12.96$, close to removing the FFN
entirely. This supports the decomposition in Eq.~\eqref{eq:split}: radial
updates control scale, while $P_u\Phi(u)$ controls direction.

Third, the tangential FFN helps preserve representation geometry under
aggregation pressure. Attention usually contracts the residual-direction
distribution, while the FFN often restores or reshapes it.
% Under amplified
% attention $\beta_A=2$, the tangential-only FFN preserves substantially more
% effective rank than radial-only or zero-FFN variants. Thus $P_u\Phi(u)$ acts as
% a non-collapsing steering mechanism when attention over-aggregates.

Fourth, the sequential-to-parallel defect experiment gives the theory a practical
consequence. The FFN input-sensitivity score $C_\ell$ identifies layers where the
sequential block
$x+A(x)+\Phi(x+A(x))$ can be safely approximated by the parallel surrogate
$x+A(x)+\Phi(x)$. Low-$C_\ell$ layers preserve quality much better than
high-$C_\ell$ layers. For GPT-2-large, parallelizing $18$ low-$C_\ell$ layers
gives an estimated $1.24\times$ speedup with only $+0.02$ loss; for Mistral,
eight low-$C_\ell$ layers give $1.10\times$ with $+0.009$ loss.

These results support the aggregation--steering view: attention provides
contextual aggregation, while the FFN supplies a local tangential steering that
steers residual directions and prevents excessive collapse. Additional
diagnostics in the appendix further support this picture. Component ablations
show that the FFN-only approximate critical candidates are already mostly
saddle-like in their local tangent spectra. Top-$k$ and hard-value experiments suggest that
attention aggregation is often dominated by a small number of value anchors.
Prefix and anchor interventions show that decoder residual geometry is
prefix-sensitive. Norm-CV measurements characterize the complementary radial
scale channel. These appendix studies refine the mechanism, but the main
evidence comes from the four experiments above.

\paragraph{Limitations.}
The theoretical mean-field results describe full attention, while causal
decoder-only models require prefix-indexed dynamics. The anti-collapse claim is
conditional, so we test it by creating aggregation pressure rather than claiming
a universal guarantee. The phase-portrait and critical-point analyses are
reduced local diagnostics, not global descriptions of the full residual space.
The simplified FFN notation is analytical; the experiments use the actual
checkpoint MLP/SwiGLU modules. Approximate-zero labels in Table~\ref{tab:ffn-restructure-critical-geometry-wide}
are local-spectrum diagnostics and should not be read as equilibrium-stability
claims without numerical root refinement. Finally, the practical experiments are
diagnostics rather than production systems: the parallelization result is an
estimated submodule-time speedup, and the gate/editing studies are not compared
against full fine-tuning or editing baselines.

\clearpage
\appendix
% Technical supplementary material integrated as the paper appendix.
% It is included from main.tex after \appendix.

\section{Supplementary Theory and Proofs}\label{app:theory-proofs}

\subsection{Notation and Main Statements}
We use the same notation as in the main paper.  Let $M=(\Sph^{d-1})^L$ and
$\Proj_u=\id-uu^\top$.  The projected aggregation--adjustment flow is
\begin{equation}
\dot u_i=\frac{1}{s_i(t)}\Proj_{u_i}\big[A_i(U)+\Phi_i(u_i)\big].
\label{supp:eq:flow}
\end{equation}
For the reduced single-particle analysis we write
\begin{equation}
F_\alpha(u)=Vu+\alpha\Phi(u),\quad
 g_\alpha(u)=\Proj_uF_\alpha(u).
\label{supp:eq:single}
\end{equation}
For a standard FFN proxy we use
\begin{equation}
\Phi(u)=\sum_p \gamma_p w_p\sigma(\langle k_p,u\rangle),
\label{supp:eq:phi}
\end{equation}
while all experiments use the actual checkpoint MLP/SwiGLU modules.

\begin{theorem}[Well-posedness; Theorem 1]\label{supp:thm:wp}
Assume every attention mask $\mathcal N(i)$ is fixed and nonempty, the softmax
temperature satisfies $\tau>0$, each local field
$\Phi_i:\Sph^{d-1}\to\R^d$ is Lipschitz, and the prescribed speed factors are
measurable with $s_i(t)\ge s_->0$ almost everywhere. Then, for every initial
condition $U_0\in M$, Eq.~\eqref{supp:eq:flow} has a unique global absolutely
continuous solution $U:\R\to M$. If the coefficients are continuous in time,
the solution is classical; if the system is autonomous, these solutions define
a global flow. ReLU activations are covered because they are globally Lipschitz.
\end{theorem}

\begin{proposition}[Stability test; Proposition 1]\label{supp:prop:stab}
Assume $\Phi$ is $C^1$ in a neighborhood of a critical residual direction
$u_\star$, where $F_\alpha(u_\star)=\rho_\star u_\star$. The intrinsic tangent
Jacobian of $g_\alpha$ at $u_\star$ is
\begin{equation}
J_\star=\Proj_{u_\star}\big(V+\alpha D\Phi(u_\star)-\rho_\star\id\big)
\big|_{T_{u_\star}\Sph^{d-1}}.
\label{supp:eq:jac}
\end{equation}
If $u_\star$ is hyperbolic, then it is locally exponentially asymptotically
stable (a residual attractor) iff $\max\RePart\Spec(J_\star)<0$.
\end{proposition}

At a fixed critical direction $u_\star$ for a fixed value of $\alpha$, set
\begin{equation}
\rho_A=\ip{u_\star}{Vu_\star},\quad
\rho_\Phi=\ip{u_\star}{\Phi(u_\star)},
\end{equation}
and, on $T=T_{u_\star}\Sph^{d-1}$, define
\begin{equation}
\begin{split}
J_A(u_\star)&=\Proj_{u_\star}(V-\rho_A\id)|_T,
\\
J_\Phi(u_\star)&=\Proj_{u_\star}(D\Phi(u_\star)-\rho_\Phi\id)|_T.
\end{split}
\end{equation}
Then
\begin{equation}
J_\star=J_A(u_\star)+\alpha J_\Phi(u_\star).
\label{eq:spectral-decomp}
\end{equation}
Along a moving critical branch $u_\alpha$, the two terms generally depend on
$\alpha$. If $u_\star$ is a common critical direction of the two fields and
$J_A(u_\star),J_\Phi(u_\star)$ are simultaneously diagonalizable over
$\mathbb C$, their common modes satisfy
$\lambda_k(\alpha)=\lambda_k^A+\alpha\lambda_k^\Phi$.

\begin{lemma}[Invariant-region and equilibrium criterion; Lemma 1]\label{supp:lem:basin}
Assume $d\ge2$, the sum in Eq.~\eqref{supp:eq:phi} is finite, the value directions $w_q$ are
unit vectors, $\sigma$ is Lipschitz, and $\alpha>0$. Fix $p$ and
$0<\Delta<\pi/2$, and set
\begin{equation}
B_p=\{u\in\Sph^{d-1}:\angle(u,w_p)\le\Delta\}.
\end{equation}
Define
\begin{equation}
\begin{split}
m_p&=\inf_{u\in\partial B_p}\gamma_p\sigma(\ip{k_p}{u}),\\
C_{pq}&=\sup_{u\in\partial B_p}|\gamma_q\sigma(\ip{k_q}{u})|,\\
M_p&=\sup_{u\in\partial B_p}\norm{\Proj_uVu}.
\end{split}
\end{equation}
If $m_p>0$ and
\begin{equation}
\alpha m_p\sin\Delta>M_p+\alpha\sum_{q\ne p}C_{pq},
\label{supp:eq:inward}
\end{equation}
then $B_p$ is forward-invariant for $\dot u=g_\alpha(u)$ and contains at least
one critical residual direction in its interior. The lemma makes no attraction
claim; stability of a hyperbolic critical direction is determined by
Proposition~\ref{supp:prop:stab}.
\end{lemma}

\begin{corollary}[Conditional anti-collapse; Corollary 1]\label{supp:cor:anti}
Let $S_1,S_2$ be disjoint Borel sets, let $0\le\theta_0\le\pi$, and let
$\mu_t$ be supported in $S_1\cup S_2$, with $\mu_t(S_1)=p$ and
$\mu_t(S_2)=1-p$. Suppose $\ip{u}{v}\le\cos\theta_0$ for every
$u\in S_1$, $v\in S_2$. Then
\begin{equation}
D(\mu_t)=1-\norm{\int u\,d\mu_t(u)}^2
\ge 2p(1-p)(1-\cos\theta_0).
\label{supp:eq:anti}
\end{equation}
If $S_1,S_2$ are forward-invariant and the initial law is supported in their
union, then the same bound holds for all times, with the initial masses $p$ and
$1-p$.
\end{corollary}

\begin{theorem}[Mean-field limit; Theorem 2]\label{supp:thm:mf}
Assume normalized time $s_i\equiv1$, homogeneous all-to-all full attention
(including self-interaction) with common
matrices $Q,K,V$, temperature $\tau>0$, and a common Lipschitz field $\Phi$.
Let
\begin{equation}
\begin{split}
\kappa(u,v)&=e^{\ip{Qu}{Kv}/\tau},\\
A[\mu](u)&=\frac{\int Vv\,\kappa(u,v)\,d\mu(v)}
{\int\kappa(u,z)\,d\mu(z)},
\end{split}
\end{equation}
and $b[\mu](u)=\Proj_u(A[\mu](u)+\alpha\Phi(u))$. For every
$\mu_0\in\Prob(\Sph^{d-1})$ there is a unique global solution
$\mu\in C([0,\infty);\Prob(\Sph^{d-1}))$ of
\begin{equation}
\partial_t\mu_t+\divS(\mu_t b[\mu_t])=0,
\quad \mu_{t=0}=\mu_0,
\label{supp:eq:mckean}
\end{equation}
equivalently $\mu_t=\mathrm{Law}(\bar u_t)$ with
$\dot{\bar u}_t=b[\mu_t](\bar u_t)$. If
$\mu_t^L=L^{-1}\sum_{i=1}^L\delta_{u_i^L(t)}$ is the empirical measure of the
corresponding $L$-particle system, then for every $T<\infty$,
\begin{equation}
\sup_{0\le t\le T}W_1(\mu_t^L,\mu_t)
\le e^{(L_u+L_\mu)T}W_1(\mu_0^L,\mu_0),
\label{eq:dobrushin}
\end{equation}
where $L_u,L_\mu$ depend only on the model parameters. Hence
$W_1(\mu_0^L,\mu_0)\to0$ implies uniform convergence on compact time intervals.
In particular, i.i.d.\ initial particles converge almost surely and in
expectation.
\end{theorem}

\begin{proposition}[Exact-flow commutator and residual-map defect; Proposition 2]
\label{supp:prop:comm}
Let $A,\Phi$ be $C^2$ vector fields on an open subset of $\R^d$, with bounded
first and second derivatives on the compact set under consideration, and let
$\varphi^X_\tau$ denote the time-$\tau$ flow of $X$. With the convention
$[A,\Phi]=D\Phi\,A-DA\,\Phi$,
\begin{align}
\varphi^\Phi_\tau\circ\varphi^A_\tau
&=\varphi^{A+\Phi}_\tau+\frac{\tau^2}{2}[A,\Phi]+O(\tau^3),
\label{supp:eq:comm}\\
\varphi^\Phi_\tau\circ\varphi^A_\tau-
\varphi^A_\tau\circ\varphi^\Phi_\tau
&=\tau^2[A,\Phi]+O(\tau^3).
\label{eq:order-defect}
\end{align}
For the residual maps
\begin{align}
S_\tau(x)&=x+\tau A(x)+\tau\Phi(x+\tau A(x)),\\
P_\tau(x)&=x+\tau A(x)+\tau\Phi(x),\\
R_\tau(x)&=x+\tau\Phi(x)+\tau A(x+\tau\Phi(x)),
\end{align}
one instead has
\begin{align}
(S_\tau-P_\tau)(x)&=\tau^2D\Phi(x)A(x)+O(\tau^3),
\label{eq:parallel-defect}\\
(R_\tau-P_\tau)(x)&=\tau^2DA(x)\Phi(x)+O(\tau^3),\\
(S_\tau-R_\tau)(x)&=\tau^2[A,\Phi](x)+O(\tau^3).
\end{align}
Thus the Lie bracket controls order reversal, whereas the direct
sequential-to-parallel defect is governed to leading order by $D\Phi\,A$ and is
not determined by the bracket alone.
\end{proposition}

\begin{theorem}[First-order steering; Theorem 3]\label{supp:thm:steer}
Assume $\Phi$ is $C^2$ near a hyperbolic critical direction $u_0$ of the
attention-only field, and let $J_0$ be its tangent Jacobian. Then there are
$\alpha_0>0$ and a unique $C^2$ branch of critical directions $u_\alpha$ near $u_0$ for
$|\alpha|<\alpha_0$, satisfying
\begin{equation}
u_\alpha=u_0-\alpha J_0^{-1}\Proj_{u_0}\Phi(u_0)+O(\alpha^2)
\label{supp:eq:steer}
\end{equation}
in ambient Euclidean norm. If $u_0$ is a residual attractor, then $u_\alpha$
remains a residual attractor for all sufficiently small $|\alpha|$.
\end{theorem}

\subsection{Proofs of Main Results}\label{app:proofs}
Throughout, $M=(\Sph^{d-1})^L$, $\Proj_u=\id-uu^\top$, and
$F(u)=Vu+\alpha\Phi(u)$ with $\Phi$ as in \eqref{supp:eq:phi}. For all $u\in\Sph^{d-1}$,
$y\in\R^d$,
\begin{equation}
\ip{u}{\Proj_u y}=\ip{u}{y}-\ip{u}{y}\norm{u}^2=0,
\label{eq:tangent}
\end{equation}
so $\Proj_u y\in T_u\Sph^{d-1}$.

\subsubsection{Proof of Theorem 1 (well-posedness)}
\begin{proof}
For $j\in\mathcal N(i)$ write
\begin{equation}
\begin{split}
a_{ij}(U)&=
\frac{\exp(\ip{Qu_i}{Ku_j}/\tau)}
{\sum_{k\in\mathcal N(i)}\exp(\ip{Qu_i}{Ku_k}/\tau)},
\\
A_i(U)&=\sum_{j\in\mathcal N(i)}a_{ij}(U)Vu_j.
\label{eq:attention-weights}
\end{split}
\end{equation}
On the compact manifold $M$, the scores satisfy
$|\ip{Qu_i}{Ku_j}|\le\norm Q\norm K$. Since each mask is nonempty,
\begin{equation}
\sum_{k\in\mathcal N(i)}\exp(\ip{Qu_i}{Ku_k}/\tau)
\ge |\mathcal N(i)|e^{-\norm Q\norm K/\tau}.
\end{equation}
Hence every $a_{ij}$ is smooth and Lipschitz on $M$, and so is
$U\mapsto A_i(U)$. The map $u\mapsto\Proj_u=\id-uu^\top$ is smooth. For the
proxy in Eq.~\eqref{supp:eq:phi}, for example,
\begin{equation}
\begin{aligned}
\norm{\Phi(u)-\Phi(v)}
&\le \mathrm{Lip}(\sigma)
\sum_p|\gamma_p|\norm{w_p}\norm{k_p}\\
&\qquad\times\norm{u-v}.
\end{aligned}
\label{eq:phi-lipschitz}
\end{equation}
Thus the spatial vector field
\begin{equation}
X_i(U)=\Proj_{u_i}\big(A_i(U)+\Phi_i(u_i)\big)
\end{equation}
is globally Lipschitz and bounded on $M$: for some $L_X,C_X<\infty$,
\begin{equation}
\begin{aligned}
\norm{X(U)-X(V)}&\le L_X\norm{U-V},\\
\norm{X(U)}&\le C_X.
\end{aligned}
\label{eq:spatial-bounds}
\end{equation}

Set $\beta_i(t)=s_i(t)^{-1}$. Then $0<\beta_i(t)\le s_-^{-1}$ almost everywhere,
and the full right-hand side
$f_i(t,U)=\beta_i(t)X_i(U)$ is measurable in $t$, continuous and Lipschitz in
$U$, with a time-uniform Lipschitz bound. In local charts on $M$, the standard
Carath\'eodory theorem therefore gives a unique absolutely continuous local
solution.

For almost every $t$, Eq.~\eqref{eq:tangent} gives
\begin{equation}
\frac{d}{dt}\norm{u_i(t)}^2
=0.
\end{equation}
Hence $M$ is invariant. Moreover,
$\norm{\dot U(t)}\le C_X/s_-$ almost everywhere. The solution remains in the
compact set $M$ and has bounded speed on every finite interval, so the standard
continuation theorem extends it globally. Uniqueness and continuous dependence
follow from Gr\"onwall's inequality. ReLU requires no separate argument because
it is globally Lipschitz. Continuity of the coefficients in time gives a
classical solution, and the autonomous case gives a global flow.
\end{proof}

\subsubsection{Proof of Proposition 1 (stability test and spectral decomposition)}
\begin{proof}
Let
$F_\alpha(u)=Vu+\alpha\Phi(u)$ and
$g_\alpha(u)=\Proj_uF_\alpha(u)=F_\alpha(u)-\ip{u}{F_\alpha(u)}u$. For any
$\delta\in\R^d$,
\begin{multline}
Dg_\alpha(u_\star)\delta
=DF_\alpha(u_\star)\delta
-\ip{\delta}{F_\alpha(u_\star)}u_\star\\
-\ip{u_\star}{DF_\alpha(u_\star)\delta}u_\star
-\ip{u_\star}{F_\alpha(u_\star)}\delta.
\label{eq:Dg}
\end{multline}
At a critical direction,
$F_\alpha(u_\star)=\rho_\star u_\star$. If
$\delta\in T_{u_\star}\Sph^{d-1}$, then $\ip{\delta}{u_\star}=0$, and
Eq.~\eqref{eq:Dg} reduces to
\begin{align}
Dg_\alpha(u_\star)\delta
&=\Proj_{u_\star}DF_\alpha(u_\star)\delta-\rho_\star\delta\\
&=\Proj_{u_\star}
\big(V+\alpha D\Phi(u_\star)-\rho_\star\id\big)\delta.
\end{align}
The right-hand side is tangent, so the restriction to
$T_{u_\star}\Sph^{d-1}$ is exactly the operator in Eq.~\eqref{supp:eq:jac}.
The linearization theorem on a finite-dimensional manifold gives local
exponential asymptotic stability when all eigenvalues have negative real part.
Conversely, if the equilibrium is hyperbolic and an eigenvalue has positive real
part, its local unstable manifold is nontrivial, so it is not an attractor.

For Eq.~\eqref{eq:spectral-decomp}, the criticality relation implies
\begin{equation}
\rho_\star=\ip{u_\star}{F_\alpha(u_\star)}
=\rho_A+\alpha\rho_\Phi.
\end{equation}
Substitution into Eq.~\eqref{supp:eq:jac} gives
$J_\star=J_A(u_\star)+\alpha J_\Phi(u_\star)$. Along a moving branch
$u_\alpha$, both summands vary with $\alpha$. If
$Vu_\star=\rho_Au_\star$ and
$\Phi(u_\star)=\rho_\Phi u_\star$, then $u_\star$ is critical for every
$\alpha$. Under simultaneous diagonalizability, a common eigenbasis makes the
diagonal entries of the sum equal to
$\lambda_k^A+\alpha\lambda_k^\Phi$.
\end{proof}

\subsubsection{Proof of Lemma 1 and Corollary 1}
\begin{proof}
\emph{Forward invariance.}
Let $r_p(u)=\angle(u,w_p)=\arccos\ip{u}{w_p}$. On $0<r_p<\pi$, its Riemannian
gradient, which is the outward unit normal to a level set of $r_p$, is
\begin{equation}
\begin{split}
\nu_p(u)=\nabla r_p(u)
=-\frac{\Proj_uw_p}{\norm{\Proj_uw_p}},
\quad
\norm{\Proj_uw_p}=\sin r_p(u).
\end{split}
\label{eq:normal}
\end{equation}
On $\partial B_p$, the contribution of the $p$-th FFN term is
\begin{equation}
\begin{split}
\ip{\nu_p(u)}
{\alpha\gamma_p\sigma(\ip{k_p}{u})\Proj_uw_p} =\\
-\alpha\gamma_p\sigma(\ip{k_p}{u})\sin\Delta \le-\alpha m_p\sin\Delta.
\end{split}
\end{equation}
The transport term satisfies
$\ip{\nu_p}{\Proj_uVu}\le\norm{\Proj_uVu}\le M_p$. For $q\ne p$, since
$\norm{\nu_p}=\norm{w_q}=1$ and $\Proj_u$ is a contraction,
\begin{equation}
\ip{\nu_p}
{\alpha\gamma_q\sigma(\ip{k_q}{u})\Proj_uw_q}
\le\alpha C_{pq}.
\end{equation}
Thus, on the boundary,
\begin{equation}
\begin{split}
\dot r_p(u)=\ip{\nu_p(u)}{g_\alpha(u)}
\le \\ M_p+\alpha\sum_{q\ne p}C_{pq}-\alpha m_p\sin\Delta<0.
\end{split}
\label{eq:strict-inward}
\end{equation}
At a first exit time from $B_p$, the outward derivative of $r_p$ would have to be
nonnegative, contradicting Eq.~\eqref{eq:strict-inward}. Hence $B_p$ is
forward-invariant.

\emph{Existence of a zero.}
The reduced field is Lipschitz on the sphere, so let $\varphi_t$ be its global
flow. For every $t>0$, forward invariance makes
$\varphi_t:B_p\to B_p$ a continuous self-map. Since
$0<\Delta<\pi/2$, $B_p$ is homeomorphic to a closed Euclidean
$(d-1)$-ball. For any sequence $t_n\downarrow0$, Brouwer's fixed-point theorem
gives $u_n\in B_p$ such that $\varphi_{t_n}(u_n)=u_n$. Passing to a subsequence,
$u_n\to u_\star\in B_p$. Boundedness of $g_\alpha$ implies
\begin{equation}
\sup_{0\le s\le t_n}\norm{\varphi_s(u_n)-u_n}\longrightarrow0.
\end{equation}
Using the integral form of the ODE,
\begin{equation}
\begin{aligned}
0
&=\frac{\varphi_{t_n}(u_n)-u_n}{t_n}\\
&=\frac1{t_n}\int_0^{t_n}
  g_\alpha(\varphi_s(u_n))\,ds\\
&\longrightarrow g_\alpha(u_\star).
\end{aligned}
\end{equation}
Therefore $g_\alpha(u_\star)=0$. Strict inwardness excludes a boundary zero, so
$u_\star$ lies in the interior. This argument does not imply attraction.

\emph{Anti-collapse.}
The cases $p=0$ or $p=1$ are immediate, so assume $0<p<1$. Let
\begin{equation}
b_j=\frac{1}{\mu_t(S_j)}\int_{S_j}u\,d\mu_t(u),
\quad \norm{b_j}\le1.
\end{equation}
Then $\int u\,d\mu_t=p b_1+(1-p)b_2$. If $U$ and $V$ are independent draws from
the conditional laws on $S_1$ and $S_2$, respectively, then
\begin{equation}
\ip{b_1}{b_2}=\E\ip{U}{V}\le\cos\theta_0.
\end{equation}
Consequently,
\begin{align}
\norm{\int u\,d\mu_t}^2
&\le p^2+(1-p)^2+2p(1-p)\cos\theta_0,
\end{align}
which is equivalent to Eq.~\eqref{supp:eq:anti}. If the regions are
forward-invariant and the initial law is supported in their union, their masses
are preserved and the estimate holds at every later time.
\end{proof}

\subsubsection{Proof of Theorem 2 (mean-field limit)}
\begin{proof}
Set
\begin{equation}
\begin{split}
N_\mu(u)=\int Vv\,\kappa(u,v)\,d\mu(v),
\\
D_\mu(u)=\int\kappa(u,v)\,d\mu(v),
\end{split}
\end{equation}
so $A[\mu]=N_\mu/D_\mu$. On the compact sphere,
\begin{equation}
\begin{split}
0<\kappa_-:=e^{-\norm Q\norm K/\tau}
\le\\\kappa(u,v)\le
\kappa_+:=e^{\norm Q\norm K/\tau}.
\end{split}
\end{equation}
Thus $D_\mu(u)\ge\kappa_-$ for every probability measure $\mu$.

For fixed $u$, the functions
$v\mapsto Vv\,\kappa(u,v)$ and $v\mapsto\kappa(u,v)$ are bounded and Lipschitz
with constants uniform in $u$. The Kantorovich-Rubinstein inequality gives
\begin{equation}
\begin{aligned}
&\norm{N_\mu(u)-N_\nu(u)}
 + |D_\mu(u)-D_\nu(u)|\\
&\qquad\le C_0W_1(\mu,\nu).
\end{aligned}
\end{equation}
Using the denominator lower bound and
\begin{equation}
\frac{N_\mu}{D_\mu}-\frac{N_\nu}{D_\nu}
=\frac{N_\mu-N_\nu}{D_\mu}
+\frac{N_\nu(D_\nu-D_\mu)}{D_\mu D_\nu}
\end{equation}
yields
\begin{equation}
\norm{A[\mu](u)-A[\nu](u)}\le L_\mu W_1(\mu,\nu).
\label{eq:measure-lip}
\end{equation}
The same bounded-derivative calculation in the first argument of $\kappa$,
together with
$\norm{\Proj_u-\Proj_v}_{\mathrm{op}}\le2\norm{u-v}$ and the Lipschitz property
of $\Phi$, gives
\begin{equation}
\norm{b[\mu](u)-b[\mu](v)}\le L_u\norm{u-v}
\label{eq:state-lip}
\end{equation}
uniformly in $\mu$.
Combining Eqs.~\eqref{eq:measure-lip} and \eqref{eq:state-lip} gives
\begin{equation}
\norm{b[\mu](u)-b[\nu](v)}
\le L_u\norm{u-v}+L_\mu W_1(\mu,\nu).
\label{eq:combined-lip}
\end{equation}

For any continuous curve $m_\cdot$ of probability measures, the
characteristic equation
\begin{equation}
\dot X_t=b[m_t](X_t)
\end{equation}
has a unique global flow on the compact sphere. The map
$m_\cdot\mapsto(X_t^m)_\#\mu_0$ is a contraction on a sufficiently short
time interval by Eq.~\eqref{eq:combined-lip} and Gr\"onwall's inequality.
Iteration gives a unique global fixed point $\mu_t$, which is the
characteristic solution of Eq.~\eqref{supp:eq:mckean}; the standard identity along
characteristics gives the continuity equation in the weak sense.

Let $\mu_t$ and $\nu_t$ be two solutions and choose an optimal initial coupling
$\pi_0$ for $W_1(\mu_0,\nu_0)$. Evolve each pair $(x,y)$ by
\begin{equation}
\dot X_t=b[\mu_t](X_t),
\quad
\dot Y_t=b[\nu_t](Y_t),
\end{equation}
and define
\begin{equation}
\pi_t=(X_t,Y_t)_\#\pi_0,
\quad
q(t)=\int\norm{X_t-Y_t}\,d\pi_0.
\end{equation}
Then $\pi_t$ is a coupling of $\mu_t$ and $\nu_t$, and the upper right
derivative of $q$ satisfies
\begin{align}
D^+q(t)
&\le\int\norm{b[\mu_t](X_t)-b[\nu_t](Y_t)}\,d\pi_0\\
&\le L_uq(t)+L_\mu W_1(\mu_t,\nu_t)\\
&\le(L_u+L_\mu)q(t).
\end{align}
Since $q(0)=W_1(\mu_0,\nu_0)$ and $W_1(\mu_t,\nu_t)\le q(t)$, Gr\"onwall gives
\begin{equation}
W_1(\mu_t,\nu_t)
\le e^{(L_u+L_\mu)t}W_1(\mu_0,\nu_0).
\label{eq:dobrushin-proof}
\end{equation}

The empirical measure of the full-attention particle system is itself a
characteristic solution of Eq.~\eqref{supp:eq:mckean} with initial datum $\mu_0^L$:
the factors $1/L$ in the empirical numerator and denominator cancel. Applying
Eq.~\eqref{eq:dobrushin-proof} with $\nu_t=\mu_t^L$ proves
Eq.~\eqref{eq:dobrushin}. If the initial particles are i.i.d.\ with law $\mu_0$,
then $W_1(\mu_0^L,\mu_0)\to0$ almost surely on the compact sphere. Since $W_1$
is uniformly bounded there, dominated convergence also gives convergence in
expectation.
\end{proof}

\subsubsection{Proof of Proposition 2 (exact-flow and residual-map defects)}
\begin{proof}
Fix $x$ and abbreviate $a=A(x)$ and $b=\Phi(x)$. For a $C^2$ vector field $X$,
Taylor expansion of its flow gives, uniformly on compact sets,
\begin{equation}
\varphi^X_\tau(x)
=x+\tau X(x)+\frac{\tau^2}{2}DX(x)X(x)+O(\tau^3).
\label{eq:flow-taylor}
\end{equation}
Applying Eq.~\eqref{eq:flow-taylor} first to $A$ and then to $\Phi$ yields
\begin{equation}
\begin{aligned}
\varphi^\Phi_\tau\circ\varphi^A_\tau(x)
&=x+\tau(a+b)\\
&\quad+\frac{\tau^2}{2}
\big(DA\,a+D\Phi\,(b+2a)\big)\\
&\quad+O(\tau^3).
\end{aligned}
\label{eq:composition-expansion}
\end{equation}
The combined field has expansion
\begin{equation}
\begin{aligned}
\varphi^{A+\Phi}_\tau(x)
&=x+\tau(a+b)\\
&\quad+\frac{\tau^2}{2}
\big(DA+D\Phi\big)(a+b)\\
&\quad+O(\tau^3).
\end{aligned}
\end{equation}
Subtracting gives Eq.~\eqref{supp:eq:comm}. Reversing the order changes the cross
term from $2D\Phi\,a$ to $2DA\,b$, which proves
Eq.~\eqref{eq:order-defect} with the stated sign convention.

For the residual maps, Taylor expansion gives
\begin{align}
\Phi(x+\tau A(x))
&=\Phi(x)+\tau D\Phi(x)A(x)+O(\tau^2),\\
A(x+\tau\Phi(x))
&=A(x)+\tau DA(x)\Phi(x)+O(\tau^2).
\end{align}
Multiplication by the outer factor $\tau$ yields
\begin{align}
(S_\tau-P_\tau)(x)&=\tau^2D\Phi(x)A(x)+O(\tau^3),\\
(R_\tau-P_\tau)(x)&=\tau^2DA(x)\Phi(x)+O(\tau^3),
\end{align}
and subtraction gives
$(S_\tau-R_\tau)(x)=\tau^2[A,\Phi](x)+O(\tau^3)$. This shows directly why the
sequential-to-parallel defect can be nonzero even when the bracket vanishes. The
exact-flow calculation is the standard Lie--Trotter local-error expansion
\citep{hairer2006geometric}.
\end{proof}

\subsubsection{Proof of Theorem 3 (first-order steering)}
\begin{proof}
Let $T_0=T_{u_0}\Sph^{d-1}$ and let
$\psi:T_0\supset\mathcal U\to\Sph^{d-1}$ be the exponential chart centered at
$u_0$, with $\psi(0)=u_0$ and $D\psi(0)=\id_{T_0}$. Define
\begin{equation}
G(v,\alpha)=\Proj_{u_0}g_\alpha(\psi(v))\in T_0.
\label{eq:steering-G}
\end{equation}
At $v=0$, the restriction of $\Proj_{u_0}$ to
$T_{\psi(v)}\Sph^{d-1}$ is the identity. Invertibility is open, so this
restriction remains an isomorphism for all sufficiently small $v$. Hence
$G(v,\alpha)=0$ iff $g_\alpha(\psi(v))=0$ near $(0,0)$.

Because $\Phi$ is $C^2$, $G$ is $C^2$. Proposition~\ref{supp:prop:stab} gives
\begin{equation}
G(0,0)=0,
\quad
D_vG(0,0)=Dg_0(u_0)|_{T_0}=J_0.
\end{equation}
Hyperbolicity implies that $J_0$ is invertible. The $C^2$ implicit function
theorem yields $\alpha_0>0$ and a unique $C^2$ curve $v(\alpha)$ with
$v(0)=0$ and $G(v(\alpha),\alpha)=0$. Put
$u_\alpha=\psi(v(\alpha))$. Differentiating at $\alpha=0$ gives
\begin{equation}
J_0v'(0)+\partial_\alpha G(0,0)=0.
\end{equation}
Since $\partial_\alpha g_\alpha(u)=\Proj_u\Phi(u)$,
\begin{equation}
\begin{aligned}
\partial_\alpha G(0,0)
&=\Proj_{u_0}\Phi(u_0),\\
v'(0)
&=-J_0^{-1}\Proj_{u_0}\Phi(u_0).
\end{aligned}
\end{equation}
Taylor's theorem and $D\psi(0)=\id$ now give
\begin{equation}
u_\alpha
=u_0-\alpha J_0^{-1}\Proj_{u_0}\Phi(u_0)+O(\alpha^2).
\end{equation}
After identifying nearby tangent spaces through the chart, the tangent Jacobian
depends continuously on $\alpha$. Therefore hyperbolicity, and the property that
all eigenvalues have negative real part, persist for sufficiently small
$|\alpha|$.
\end{proof}

\section{Experimental Details}\label{app:experimental-details}
\sloppy
This subsection reports the details needed to reproduce every reported number. All statistics are
computed on OpenWebText \citep{gokaslan2019openwebtext} over GPT-2 and GPT-2-large
\citep{radford2019language}, Pythia-410M and Pythia-1.4B \citep{biderman2023pythia},
Mistral-7B \citep{jiang2023mistral}, and Llama-3-8B \citep{dubey2024llama}. The main
text reports four experiment groups: (i)~one-step angular prediction, (ii)~radial and
tangential FFN ablations, (iii)~diversity under aggregation pressure, and
(iv)~commutator-guided approximate parallelization.
 
\paragraph{Data and pre-processing.}
Text is drawn from OpenWebText \citep{gokaslan2019openwebtext}, loaded through the
HuggingFace hub (\texttt{Skylion007/\allowbreak openwebtext}) by streaming. For each
model we concatenate non-empty documents and tokenize a fixed contiguous window with
the model's native tokenizer; no additional cleaning, normalization, or filtering is
applied. The loader prints the resolved data source at run time so that a silent
fallback to any placeholder corpus cannot occur. Weight-only diagnostics use no data.
All pre-processing code is included in the released code appendix.
 
\paragraph{Models.}
The six checkpoints span scale and architecture: GPT-2 (124M) and GPT-2-large (774M)
(learned-absolute positions, LayerNorm, dense attention); Pythia-410M and Pythia-1.4B
(RoPE, LayerNorm, parallel attention/FFN blocks); and Mistral-7B and Llama-3-8B
(RoPE, RMSNorm, grouped-query attention). Checkpoints are loaded through
TransformerLens/HuggingFace; the $7$--$8$B models are loaded without weight processing
and evaluated in bfloat16.

\paragraph{Hyperparameters: ranges and selection criterion.}
The interventions have a small number of controls, listed with the range swept in
Table~\ref{tab:hparams}. We do \emph{not} tune these to maximize a score: the trained
operating point $(\beta_A,\alpha_F)=(1,1)$ is fixed by the pretrained model and used
as the reference against which all other settings are compared, and the remaining
grids (edit sizes, ablation sizes, layer depth) are swept to trace a response curve
rather than to select a best value. Where a single value is needed for a headline
number we use the reference point $(1,1)$, the largest ablation set $k{=}512$, and the
mid-range edit size; these choices are stated with each result.

\begin{table}[t]\centering\small
\caption{Controls varied in the experiments and the ranges swept. The trained point
$(\beta_A,\alpha_F){=}(1,1)$ is the reference, not a tuned selection.}
\label{tab:hparams}
\begin{tabular}{@{}lll@{}}
\toprule
Control & Symbol & Values swept \\
\midrule
Attention gain      & $\beta_A$   & $\{0,0.5,1,1.5,2\}$ \\
FFN gain            & $\alpha_F$  & $\{0,0.5,1,1.5,2\}$ \\
Edit size           & $s$         & $\{0.5,1,2,4,8\}$ \\
Ablation set size   & $k$         & $\{32,128,512\}$ \\
Neurons per edit    & $n$         & $32$ \\
Intervention depth  & $\ell/L$    & $\{0,\tfrac1{16},\dots,1\}$ \\
Softmax temperature & $\tau$      & model default \\
\bottomrule
\end{tabular}
\end{table}

\paragraph{Computing infrastructure.}
Experiments run on a single multi-GPU node with NVIDIA GPUs. Models up to
$1.4$B run on one GPU in float32; the $7$--$8$B models are sharded across two GPUs
and run in bfloat16. All pretrained-model runs are inference-only (no fine-tuning),
except the optional anti-collapse regularizer probe. The software stack is
Python~$3.10{+}$ with PyTorch, TransformerLens, and HuggingFace
\texttt{transformers} and \texttt{datasets}; the exact versions are pinned in the
released \texttt{requirements.txt} and an \texttt{environment.yml}. A full pass over
the six models for the core diagnostics completes within minutes to a few hours per
model, dominated by the $7$--$8$B forward passes.

\paragraph{Randomness and seeds.}
Weight-only diagnostics (OV skew/sym ratio, dominant-direction fraction) are
deterministic functions of the checkpoint and independent of seed and data.
Data-dependent diagnostics use a single fixed evaluation window per model. Wherever a
random baseline appears -- random FFN-channel ablation and random edit
directions -- the draws come from \texttt{numpy.random.\allowbreak default\_rng(seed)} with a fixed
seed, and the reported random baseline is the mean over five seeds. PyTorch global
seeds are set before the forward passes; with fixed data and seeds the pipeline is
deterministic up to the nondeterminism of GPU reductions.

\paragraph{Evaluation metrics.}
All metrics are defined formally in the main text and appendix; we restate them and
their motivation here.
\begin{itemize}\setlength\itemsep{1pt}
\item \emph{One-step faithfulness} -- the gap
$\cos(A{+}\Phi,\text{actual})-\cos(A,\text{actual})$ between the cosine of the true
angular update with the attention-plus-FFN prediction and with attention alone;
it measures whether the FFN reaction term is needed to explain the layer's motion.
\item \emph{Edit transfer} -- the $R^2$ of the induced target-logit change against edit
size $s$ and the sign of the slope, testing the linearity and direction predicted by
the steering law.
\item \emph{Editing efficiency and locality} -- the target-logit change per unit edit
norm, and the held-out next-token KL divergence, comparing the FFN value path to
equal-norm query/key and random edits.
\item \emph{Commutator defect} -- $C_\ell=\lVert[a,b]\rVert/(\lVert a\rVert\,\lVert b\rVert)$,
the normalized Lie bracket of the tangential attention and FFN fields, measuring
whether their order matters.
\item \emph{Directional diversity} -- $D=1-\lVert\mathbb{E}\,u\rVert^2$, the spread of
token directions.
\item \emph{Norm concentration} -- the within-layer coefficient of variation of the
residual norm on the bulk of tokens, justifying the directional description.
\item \emph{OV non-symmetry} -- the ratio
$\lVert\mathrm{skew}(V)\rVert_F/\lVert\mathrm{sym}(V)\rVert_F$ and the fraction of
heads with a positive dominant eigenvalue.
\end{itemize}

\paragraph{Number of runs.}
Because the weight-only and single-window diagnostics are deterministic, each such
number is a single run. Random baselines (random ablation, random edits) are reported
as the mean over five seeds. Every diagnostic is computed independently for all six
models, so each cross-model claim aggregates six runs (one per model), and each
per-layer curve aggregates all layers of a model.

\paragraph{Analysis beyond point summaries.}
We report distributions rather than single scalars wherever possible: per-layer
curves across the full depth of each model (one-step gap, $C_\ell$, norm, diversity,
norm CV), per-head statistics for the OV analysis (fraction of heads over all heads
and layers), and per-model small multiples so that variation across scale and
architecture is visible directly. For the editing transfer we report the full
logit-vs-$s$ response with its linear fit rather than a single slope.

\paragraph{Statistical significance.}
Our primary robustness argument is qualitative consistency: the sign of the effect is
identical across all six models and, for the OV analysis, across essentially $100\%$
of heads. We do not yet apply formal significance tests (e.g.\ Wilcoxon signed-rank)
to the cross-model comparisons; because the effects are directionally uniform across
independent models and architectures, we treat this uniformity as the evidence of
robustness and note formal testing as future work.

\paragraph{Final hyperparameters.}
The settings used for the headline numbers are: operating point
$(\beta_A,\alpha_F)=(1,1)$; edit sizes $s\in\{0.5,1,2,4,8\}$ with efficiency reported
at the smallest $s$ and transfer fit over the full range; $n=32$ value neurons per
edit; ablation size $k=512$ for the headline localization numbers (with $k\in\{32,128,512\}$
swept); intervention applied per layer across the full depth; and the model-default
softmax temperature.

\section{Supplementary Experimental Results}\label{app:diagnostics}\label{app:supplementary-results}
This appendix reports diagnostics that support the aggregation--adjustment
interpretation but are not part of the four main experiments.  The main text
focuses on angular velocity prediction, tangential/radial FFN ablation,
diversity under aggregation pressure, and commutator-guided parallelization.
Here we collect additional controls: FFN gain and gate calibration, hard-value
aggregation, critical-geometry ablations, phase-portrait diagnostics, prefix and
anchor interventions, editing locality, OV non-symmetry, and radial norm
statistics.

%  --  --  --  --  --  --  --  --  --  --  --  --  --  --  --  --  --  --  --  --  --  --  -- 
\subsection{Gain corridor, FFN gates, and practical controls}\label{app:gain-gates-controls}
Scaling the FFN output tests whether pretrained checkpoints operate near a
stable aggregation--adjustment corridor.  We evaluate updates of the form
\begin{equation}
x_{\ell+1}
=
x_\ell
+
\beta_A A_\ell(x_\ell)
+
\alpha_F\Phi_\ell(x_\ell),
\end{equation}
with $\beta_A=1$ for the gain sweep.  Removing the FFN sharply increases loss,
while the best operating point remains close to the pretrained scale
$\alpha_F=1$.  This supports the view that the FFN scale is calibrated during
pretraining rather than arbitrary.

We also freeze the checkpoint and learn only scalar layerwise FFN gates
$\alpha_\ell$ in
\begin{equation}
x_{\ell+1}
=
x_\ell
+
A_\ell(x_\ell)
+
\alpha_\ell\Phi_\ell(x_\ell).
\end{equation}
The learned gates are nontrivial and architecture-dependent, but the gains over
the natural $\alpha_\ell=1$ baseline are small and not universal.  We therefore
use them as sensitivity diagnostics rather than as a competitive adaptation
method.

\begin{table}[H]
\centering
\small
\caption{\textbf{Layerwise FFN gate diagnostics.}  All checkpoint weights are
frozen and only scalar FFN gates $\alpha_\ell$ are learned.  The table reports
mean learned gate, layerwise range, and loss difference relative to the natural
baseline $\alpha_\ell=1$.}
\label{tab:supp-ffn-gates}
\begin{tabular}{lccc}
\toprule
Model & Mean $\alpha_\ell$ & Range of $\alpha_\ell$ & $\Delta L$ vs. $\alpha_\ell=1$ \\
\midrule
GPT-2        & $0.852$ & $0.667$--$0.950$ & $-0.028$ \\
GPT-2-large  & $0.917$ & $0.655$--$1.119$ & $-0.005$ \\
Pythia-410M  & $1.063$ & $0.709$--$1.412$ & $+0.077$ \\
Pythia-1.4B  & $1.048$ & $0.760$--$1.207$ & $+0.017$ \\
Mistral-7B   & $1.009$ & $0.861$--$1.098$ & $+0.006$ \\
\bottomrule
\end{tabular}
\end{table}

\begin{figure*}[t]
\centering
\includegraphics[width=0.90\textwidth]{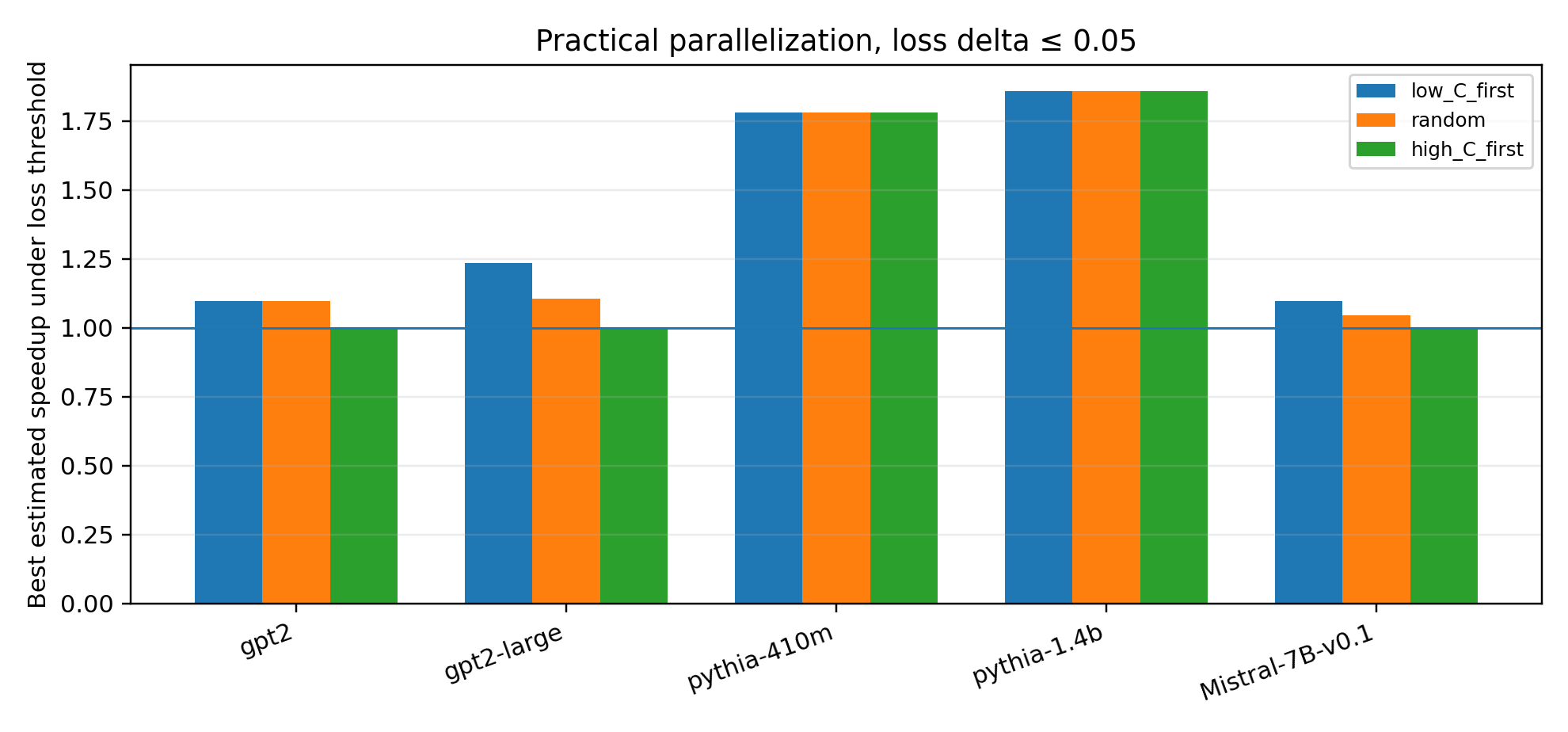}
\caption{\textbf{Safe speedup at loss threshold $0.05$.} For GPT-2,
GPT-2-large and Mistral, low-$C_\ell$ selection gives the best safe speedup.
Pythia curves are flat, consistent with its parallel-residual design: attention
and FFN already read the same input.}
\label{fig:supp-safe-speedup}
\end{figure*}

\begin{figure*}[t]
\centering
\includegraphics[width=0.90\textwidth]{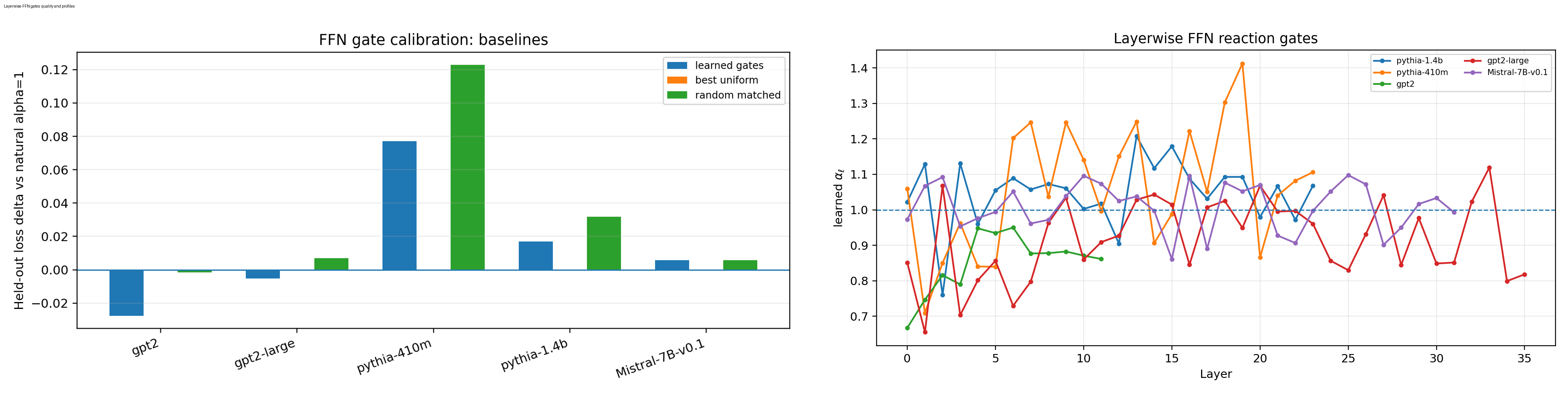}
\caption{\textbf{FFN gate calibration.} Learned layerwise FFN gates are
nontrivial and architecture-dependent, but improvements over the natural
pretrained scale are small.}
\label{fig:supp-practical-gates}
\end{figure*}

\begin{figure*}[t]
\centering
\includegraphics[width=0.94\textwidth]{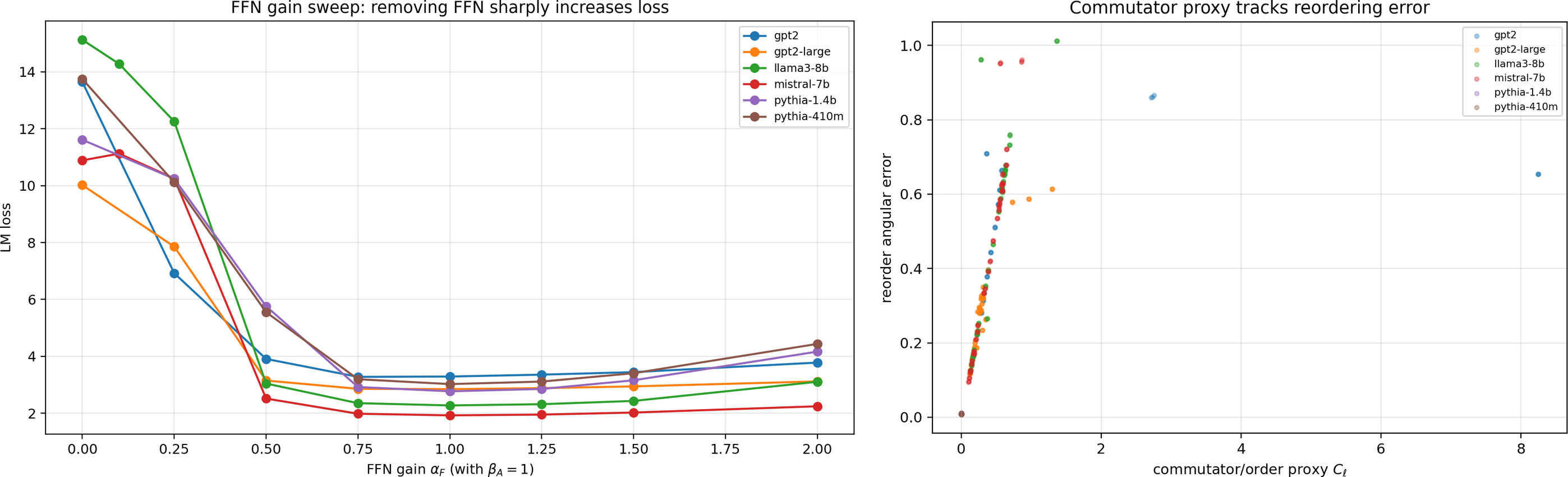}
\caption{\textbf{Gain and commutator controls.} Left: FFN gain sweep with
attention fixed. Right: commutator/order proxy versus reordering angular error.}
\label{fig:supp-control-results}
\end{figure*}

%  --  --  --  --  --  --  --  --  --  --  --  --  --  --  --  --  --  --  --  --  --  --  -- 
\subsection{Hard-value and \texorpdfstring{top-$k$}{top-k} aggregation diagnostics}\label{app:hard-value}
The main attention field uses the full soft aggregation
\begin{equation}
A_i^{\mathrm{soft}}
=
\sum_j a_{ij}Vu_j.
\end{equation}
As a sparsity diagnostic, we replace this by the top-$k$ value contributions,
where
\begin{equation}
S_i^{(k)}
=
\operatorname{TopK}_j
\left\|a_{ij}Vu_j\right\|,
\end{equation}
and
\begin{equation}
A_i^{\mathrm{top}k}
=
\sum_{j\in S_i^{(k)}} a_{ij}Vu_j.
\end{equation}
This tests whether the soft attention update is genuinely distributed or already
dominated by a small number of value anchors.

\begin{table*}[t]
\centering
\small
\caption{\textbf{Top-$k$ value aggregation diagnostic.}  Cosine is measured
between the top-$k$ reconstruction plus FFN and the full soft attention plus FFN
field.  Loss increase is measured relative to full soft attention.  Top-$16$
nearly recovers both direction and loss in all valid GPT/Pythia models.}
\label{tab:supp-hard-value-topk}
\begin{tabular}{lcccccccc}
\toprule
 & \multicolumn{4}{c}{Cosine to full soft direction} & \multicolumn{4}{c}{Loss increase} \\
\cmidrule(lr){2-5}\cmidrule(lr){6-9}
Model & top-1 & top-4 & top-16 & top-32 & top-1 & top-4 & top-16 & top-32 \\
\midrule
GPT-2        & $0.902$ & $0.954$ & $0.984$ & $0.990$ & $+0.0713$ & $+0.0295$ & $+0.0070$ & $+0.0036$ \\
GPT-2-large  & $0.909$ & $0.957$ & $0.988$ & $0.995$ & $+0.0125$ & $+0.0058$ & $+0.0011$ & $+0.0001$ \\
Pythia-410M  & $0.933$ & $0.972$ & $0.994$ & $0.998$ & $+0.0696$ & $+0.0247$ & $+0.0041$ & $+0.0013$ \\
Pythia-1.4B  & $0.917$ & $0.967$ & $0.994$ & $0.998$ & $+0.0858$ & $+0.0274$ & $+0.0042$ & $+0.0011$ \\
\bottomrule
\end{tabular}
\end{table*}

The rapid saturation indicates that residual attention updates in these models
are often value-anchor dominated rather than fully distributed.  This does not
replace the soft attention model, but identifies an important sparse regime of
the same aggregation mechanism.

%  --  --  --  --  --  --  --  --  --  --  --  --  --  --  --  --  --  --  --  --  --  --  -- 
\subsection{Component ablation of critical residual geometry}\label{app:component-geometry}
To isolate how the FFN reorganizes residual directions, we compare three reduced
fields:
\begin{equation}
g_{\mathrm{attn}}(u)=P_uVu,
\end{equation}
\begin{equation}
g_{\mathrm{FFN}}(u)=P_u\Phi(u),
\end{equation}
and
\begin{equation}
g_{\mathrm{full}}(u)=P_u\left(Vu+\Phi(u)\right).
\end{equation}
For each field, we search for candidate critical residual directions
$u_\star$ satisfying $g(u_\star)\approx0$, keeping only candidates with
\begin{equation}
\frac{\|g(u_\star)\|}{\|F(u_\star)\|+\varepsilon}
\leq
0.05.
\end{equation}
We classify candidates by the tangent spectrum of the local Jacobian.

\begin{table*}[t]
\centering
\small
\caption{\textbf{Component ablation of critical residual geometry.}
We compare Attention-only $P_uVu$, FFN-only $P_u\Phi(u)$, and the full field
$P_u(Vu+\Phi(u))$.  Each entry reports A/S/R, the number of attractor, saddle,
and repelling candidates, followed by the saddle fraction in parentheses.}
\label{tab:supp-component-ablation-wide}
\begin{tabular}{lccc}
\toprule
Model & Attention-only & FFN-only & Full \\
\midrule
GPT-2
& $6/25/3$ \;($73.5\%$), $n=34$
& $3/14/4$ \;($66.7\%$), $n=21$
& $3/13/5$ \;($61.9\%$), $n=21$ \\
GPT-2-large
& $2/17/3$ \;($77.3\%$), $n=22$
& $2/32/3$ \;($86.5\%$), $n=37$
& $2/30/4$ \;($83.3\%$), $n=36$ \\
Pythia-410M
& $2/21/3$ \;($80.8\%$), $n=26$
& $1/30/5$ \;($83.3\%$), $n=36$
& $1/29/6$ \;($80.6\%$), $n=36$ \\
Pythia-1.4B
& $5/10/3$ \;($55.6\%$), $n=18$
& $2/42/2$ \;($91.3\%$), $n=46$
& $2/44/1$ \;($93.6\%$), $n=47$ \\
\midrule
All models
& $15/73/12$ \;($73.0\%$), $n=100$
& $8/118/14$ \;($84.3\%$), $n=140$
& $8/116/16$ \;($82.9\%$), $n=140$ \\
\bottomrule
\end{tabular}
\end{table*}

The FFN-only field is already strongly saddle-dominated: aggregated across the
four models, $118/140$ FFN-only candidates are saddle-like.  The full field
remains similarly saddle-dominated, with $116/140$ saddle-like candidates.  Thus
FFN adjustment should not be interpreted as merely creating stable residual
attractors; it reorganizes the local phase portrait through saddle/transport-like
geometry.

%  --  --  --  --  --  --  --  --  --  --  --  --  --  --  --  --  --  --  --  --  --  --  -- 
\subsection{Phase reconstruction, saddle spectra, and basin maps}\label{app:phase-saddle}
We reconstruct an empirical field
$\hat G_\ell(u)\approx P_u(A_\ell+\Phi_\ell)$ in sampled PCA subspaces.  The goal
is not to identify exact global attractors, but to diagnose whether local fields
look convergent, saddle-like, or rotational.  Inward scores are close to
balanced, while estimated spectra often have positive real and nonzero imaginary
parts.  This suggests that real residual flows are transient and transport-like
rather than simple autonomous convergence to stable residual attractors.

\begin{figure*}[t]
\centering
\includegraphics[width=0.86\textwidth]{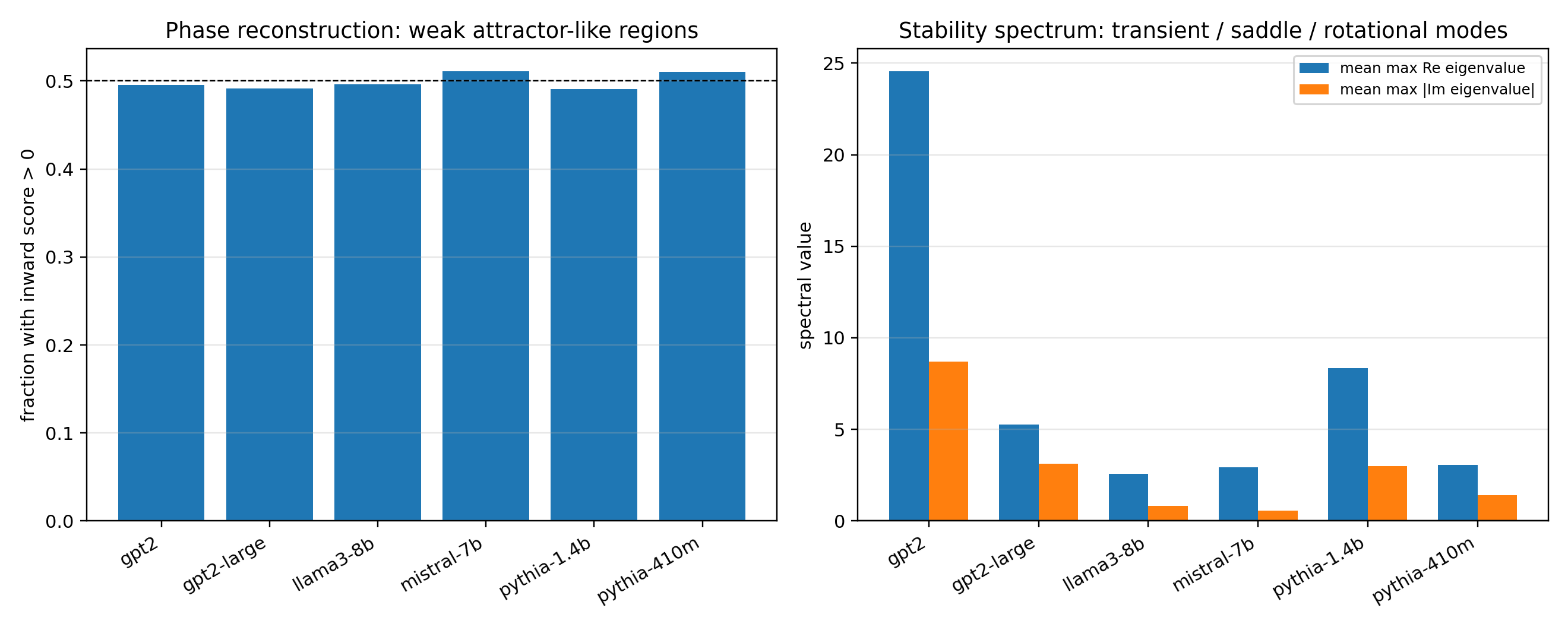}
\caption{\textbf{Phase reconstruction and local stability.} The empirical phase
portrait shows weak attractor-like regions, but the stability spectrum is often
saddle/rotational.}
\label{fig:supp-phase-stability}
\end{figure*}

For the original full field $g_\alpha(u)=P_u(Vu+\alpha\Phi(u))$, saddle residual
directions dominate the recovered critical set at $\alpha=1$: across GPT-2,
GPT-2-large, Pythia-410M, and Pythia-1.4B, $91/111$ candidates are saddle-like,
$11$ are repelling, and $9$ are residual attractors.  This supports the view that
FFN adjustment reshapes saddle and basin geometry rather than simply
manufacturing stable residual attractors.

\begin{figure*}[t]
\centering
\includegraphics[width=0.86\textwidth]{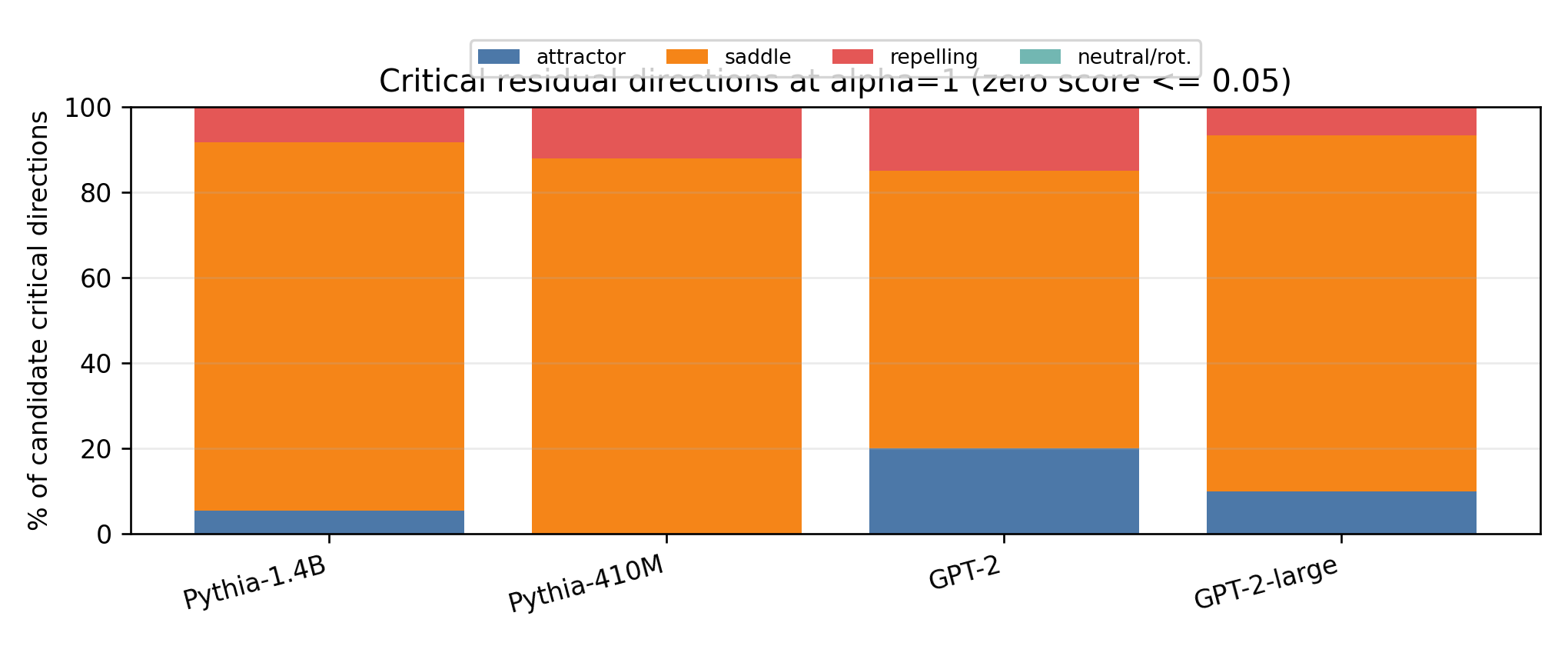}
\caption{\textbf{Critical residual direction classes.} Among candidate critical
residual directions at $\alpha=1$ and zero score $\leq0.05$, saddle residual
directions are the dominant class in all tested accessible models.}
\label{fig:supp-saddle-class}
\end{figure*}

Varying $\alpha$ shifts the tangent spectra: maximal real parts, saddle index,
and imaginary components change with FFN gain.  Two-dimensional basin slices
further show that turning on the FFN can merge, shift, or fragment local basin
regions.  These maps are qualitative PCA-sphere diagnostics, not global phase
portraits of the full model.

\begin{figure*}[t]
\centering
\includegraphics[width=0.96\textwidth]{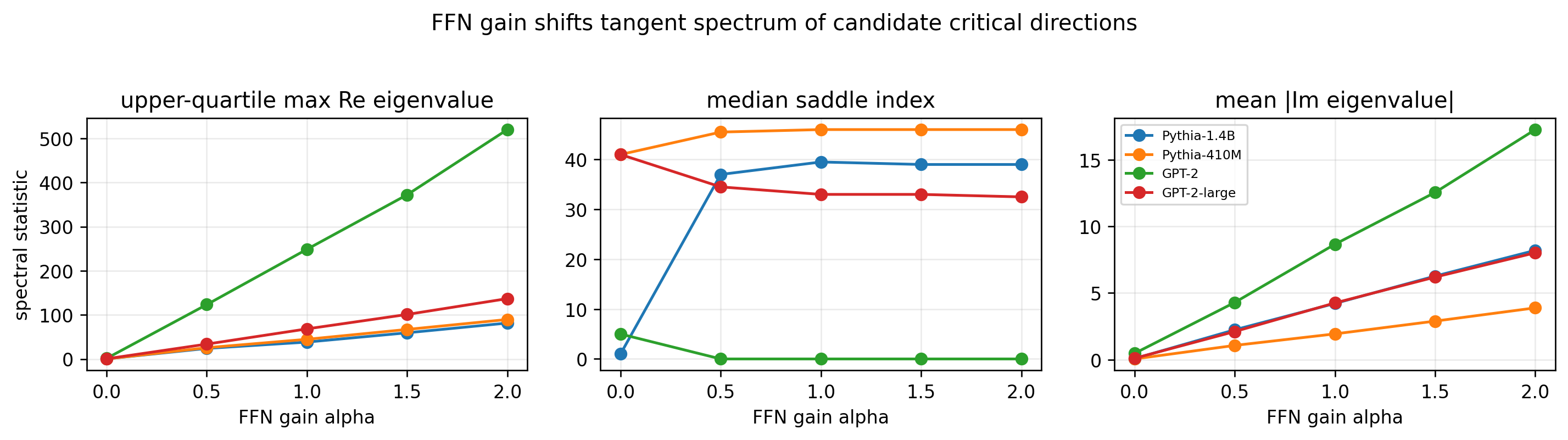}
\caption{\textbf{FFN gain shifts tangent spectra.} Continuation in $\alpha$
changes maximal real parts, saddle index, and imaginary components of candidate
critical residual directions.}
\label{fig:supp-saddle-spectral}
\end{figure*}

\begin{figure*}[t]
\centering
\includegraphics[width=0.90\textwidth,height=0.78\textheight,keepaspectratio]{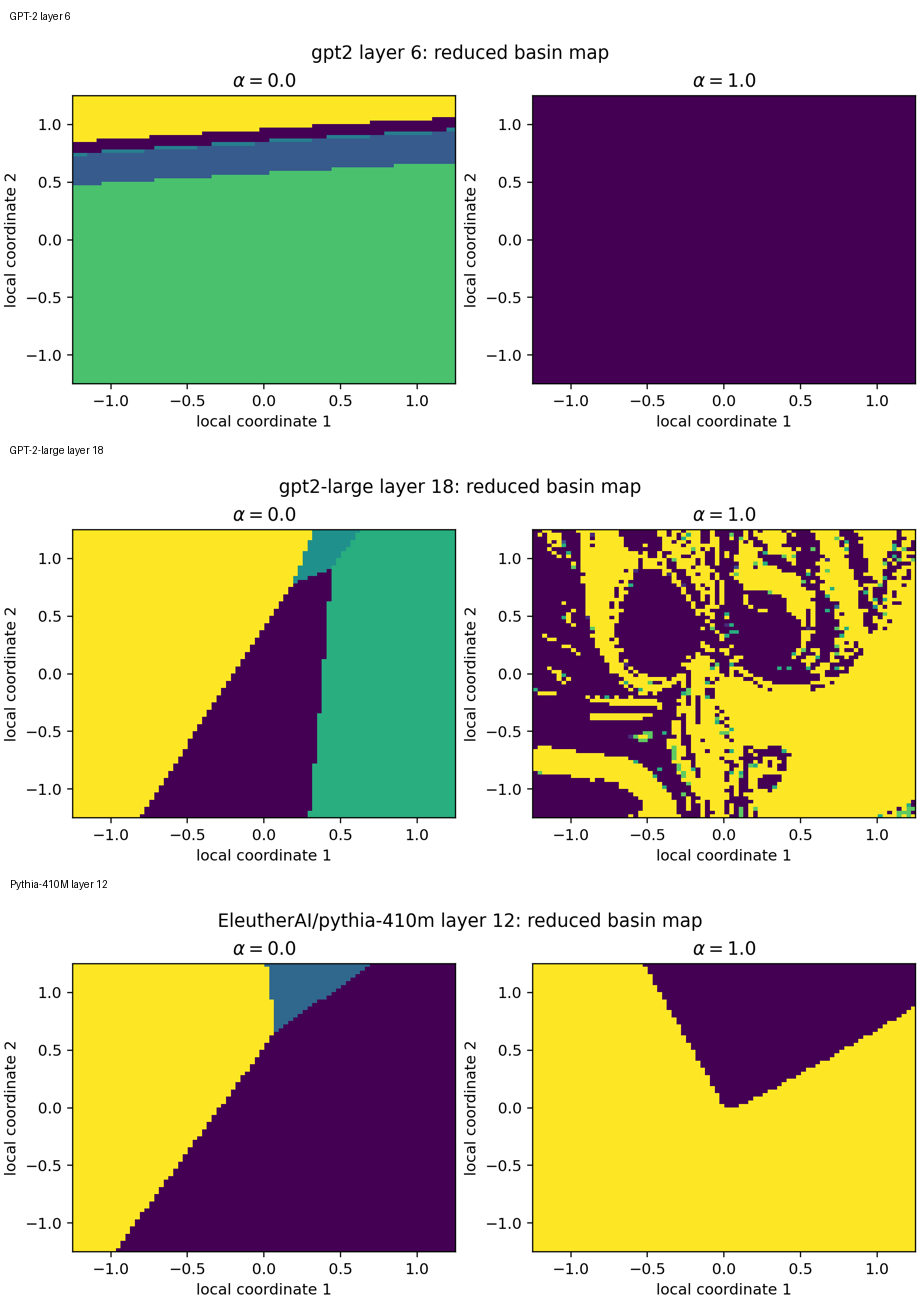}
\caption{\textbf{Basin maps under FFN gain.} Comparing $\alpha=0$ and $\alpha=1$
in two-dimensional sphere slices shows that FFN adjustment can move basin
boundaries and change local attraction regions.}
\label{fig:supp-saddle-basins}
\end{figure*}

%  --  --  --  --  --  --  --  --  --  --  --  --  --  --  --  --  --  --  --  --  --  --  -- 
\subsection{Causal prefix, anchor, and editing diagnostics}\label{app:prefix-anchor-editing}
Decoder-only attention is causal, so residual geometry should depend on the
prefix rather than only on an unordered token distribution.  We test this by
fixing a suffix and varying the prefix, then measuring the angular spread of the
last-token residual direction.  Large spread supports a prefix-indexed,
Volterra-like view of decoder dynamics.

Anchor interventions provide a local controllability diagnostic.  We identify
prefix tokens with large tangential FFN adjustment norm $\|P_u\Phi(u)\|$, replace
or perturb them, and measure the downstream last-token angular shift.  The shift
is measurable but moderate, indicating that high-reaction prefix tokens influence
but do not determine the downstream basin.

\begin{figure*}[t]
\centering
\includegraphics[width=0.88\textwidth]{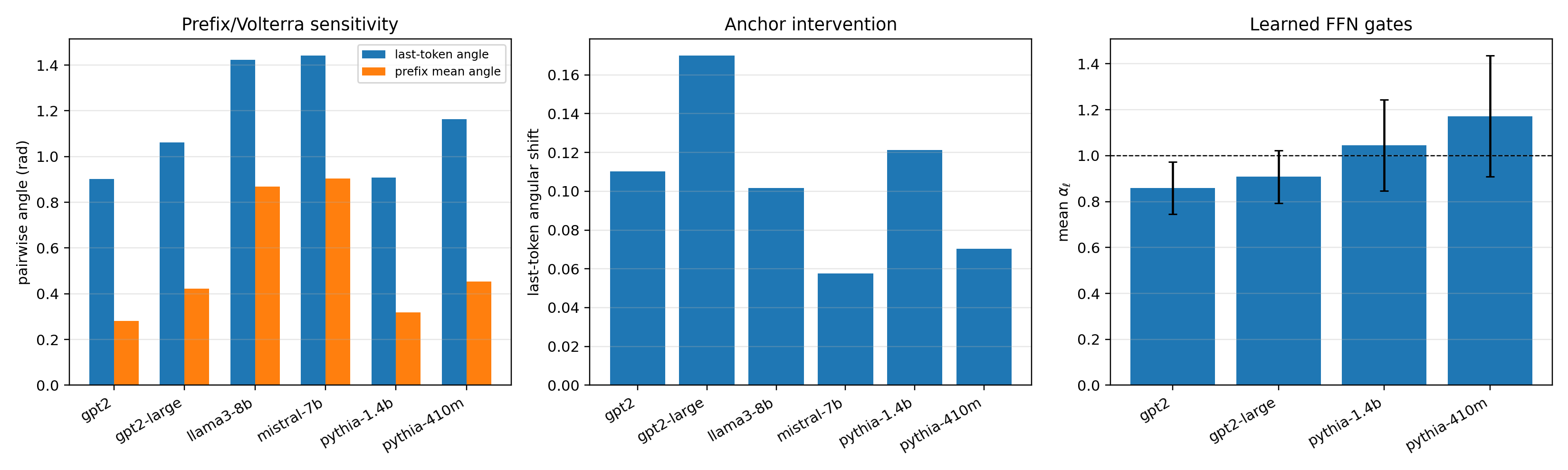}\\[-1mm]
\includegraphics[width=0.48\textwidth]{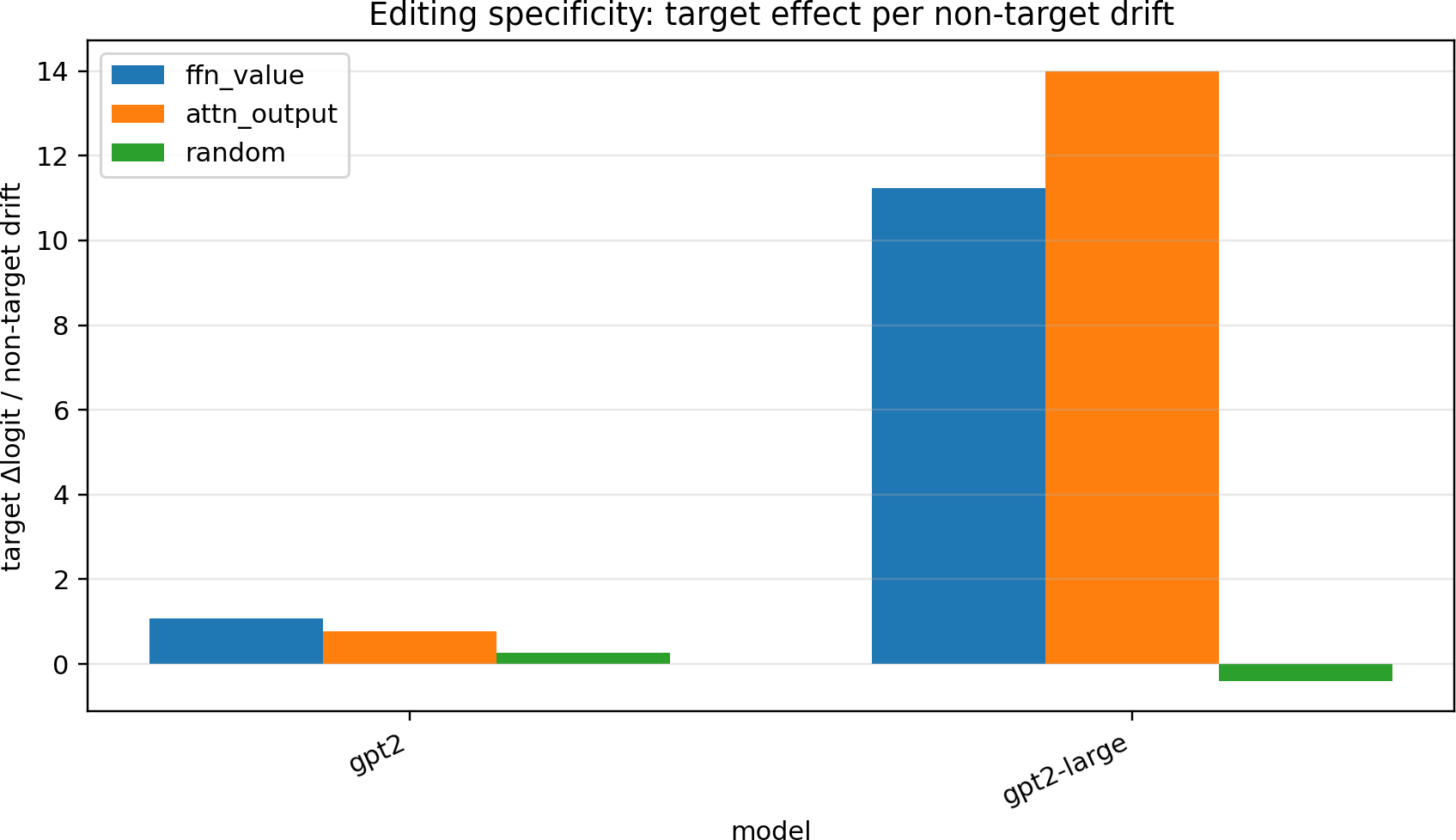}
\caption{\textbf{Prefix, anchor, gates, and editing diagnostics.}
Top: different prefixes send the same suffix to different last-token residual
directions; replacing an FFN-reactive anchor token causes a measurable downstream
angular shift; learned FFN gates are layer- and architecture-dependent. Bottom:
editing specificity, measured as target movement normalized by non-target logit
drift. FFN-value edits are local and target-aligned, while random directions are
ineffective.}
\label{fig:supp-prefix-anchor-editing}
\end{figure*}

FFN-value edits are nearly linear and low-drift on GPT-2/GPT-2-large.  Broad
attention-output edits can induce larger raw logit movement but also produce
larger distributional drift.  Editing is therefore treated as supporting
evidence for local steering, not as the main practical application.

%  --  --  --  --  --  --  --  --  --  --  --  --  --  --  --  --  --  --  --  --  --  --  -- 
\subsection{Radial norms, OV non-symmetry, and architectural diagnostics}\label{app:radial-ov-atlas}
The angular dynamics are speed-regulated by residual magnitude, so radial
statistics are complementary to the tangential results.  We report layerwise
coefficients of variation
\begin{equation}
\operatorname{CV}_\ell
=
\frac{\operatorname{std}(\|x^\ell\|)}{\operatorname{mean}(\|x^\ell\|)},
\end{equation}
after removing the largest $5\%$ norm outliers.  The x-axis is the relative
layer index $\ell/L$.

Real OV maps are far from symmetric, which motivates using non-symmetric
transport and tangent spectra rather than a purely symmetric energy picture.
Magnitude/diversity atlas measurements show that residual norms grow across
depth and that bulk diversity remains architecture-dependent.

\begin{figure*}[t]
\centering
\includegraphics[width=0.80\textwidth]{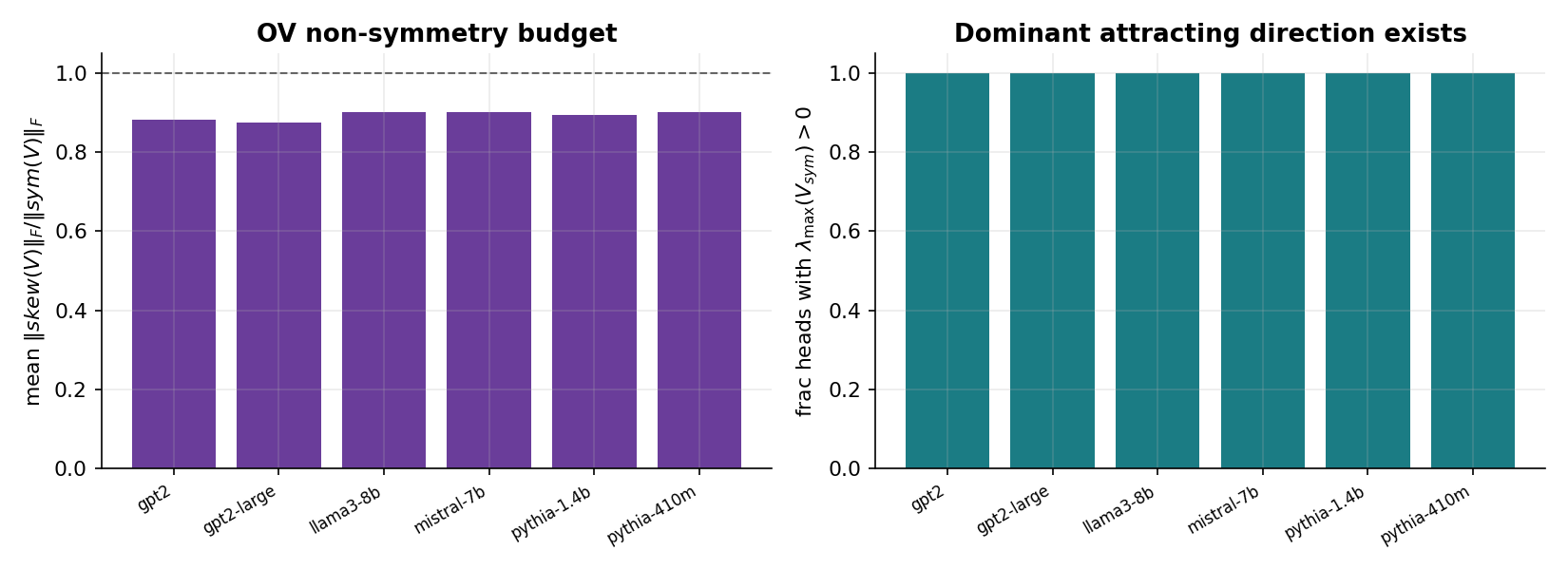}
\caption{\textbf{OV non-symmetry and dominant directions.} Real OV maps are far
from symmetric, while dominant directions are present essentially everywhere.}
\label{fig:supp-ov}
\end{figure*}

\begin{figure*}[t]
\centering
\includegraphics[width=0.84\textwidth]{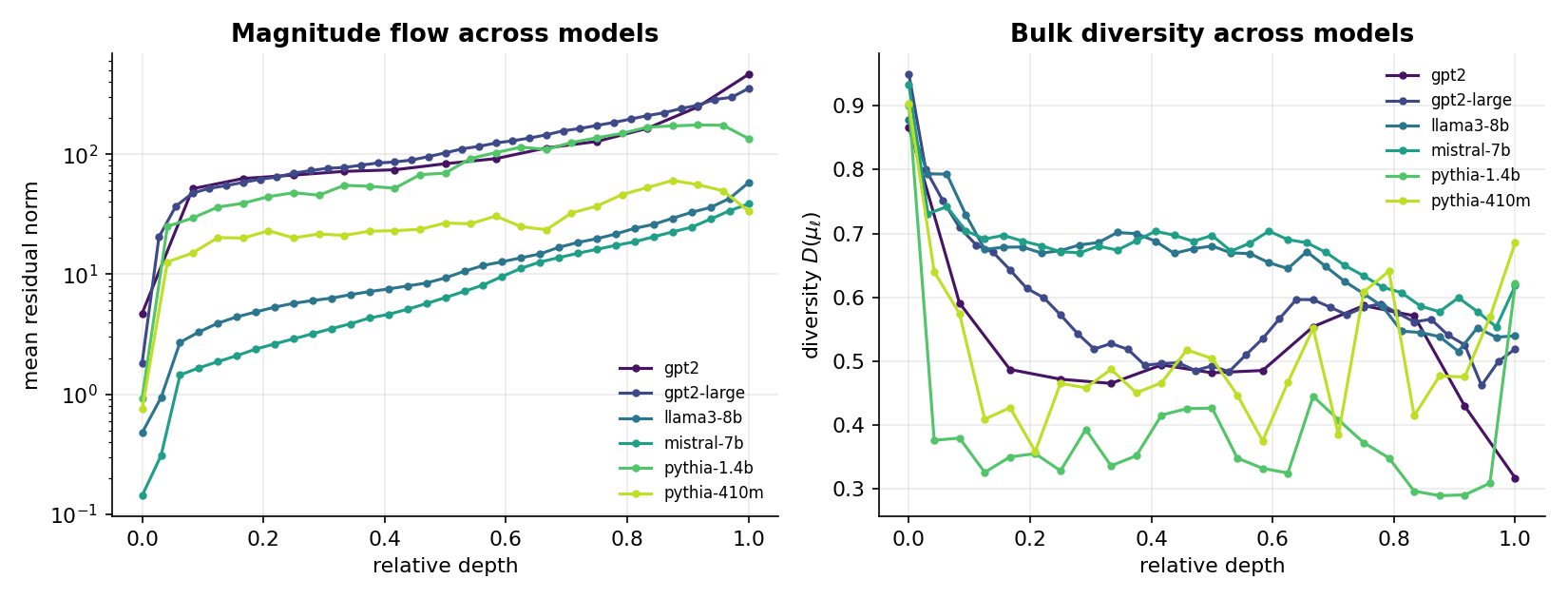}
\caption{\textbf{Magnitude and diversity atlas.} Residual norms grow across
depth, justifying speed-regulated angular dynamics. Bulk diversity remains
nonzero but is architecture-dependent.}
\label{fig:supp-atlas}
\end{figure*}

\begin{figure*}[t]
\centering
\includegraphics[width=0.92\textwidth]{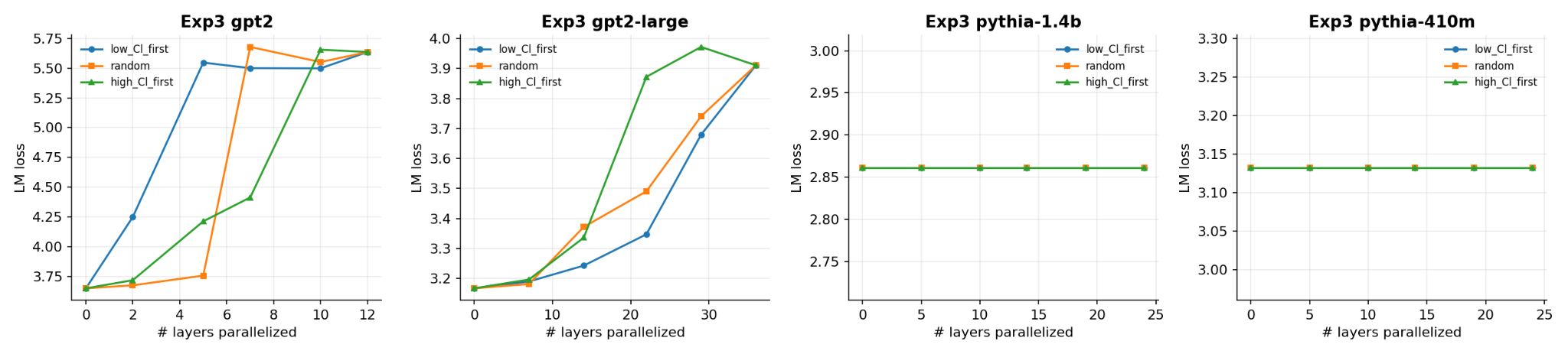}
\caption{\textbf{Commutator-guided parallelization.} Layers are parallelized in
low-$C_\ell$-first, random, or high-$C_\ell$-first order. The ordering matters
most when the model has a nontrivial commutator profile; flat curves indicate
layers or architectures insensitive to this intervention.}
\label{fig:supp-parallel}
\end{figure*}
\FloatBarrier

\bibliographystyle{plainnat}
\bibliography{references}

\end{document}